%% file: main.tex
\documentclass[conference]{IEEEtran}
\usepackage{times}

\usepackage{cite}
\usepackage{multicol}
\usepackage[bookmarks=true, colorlinks]{hyperref}
\usepackage{graphicx}
\usepackage{subfigure}
\usepackage{amsmath}
\usepackage{amssymb}
\usepackage{float}
\usepackage{multirow}
\usepackage{multicol}
\usepackage[linesnumbered,ruled]{algorithm2e}
\SetKw{Continue}{continue}
\SetKw{Break}{break}
\SetKw{Not}{not}

\newtheorem{theorem}{Theorem}
\newtheorem{lemma}{Lemma}

\let\oldemptyset\emptyset
\let\emptyset\varnothing

\newcommand{\ie}{\textit{i.e.}}

\newcommand\inv[1]{#1\raisebox{1.15ex}{$\scriptscriptstyle-\!1$}}

\let\emptyset\varnothing

\begin{document}

\title{Motion Planning (In)feasibility Detection using a Prior Roadmap via Path and Cut Search}

\author{\authorblockN{Yoonchang Sung\textsuperscript{1} and Peter Stone\textsuperscript{1,2}}
\authorblockA{
\textsuperscript{1}Department of Computer Science, The University of Texas at Austin, USA\\
Email: \{yooncs8, pstone\}@cs.utexas.edu\\
\textsuperscript{2}Sony AI}
}

\maketitle

%===============================================================================

\begin{abstract}
Motion planning seeks a collision-free path in a configuration space (C-space), representing all possible robot configurations in the environment. As it is challenging to construct a C-space explicitly for a high-dimensional robot, we generally build a graph structure called a \emph{roadmap}, a discrete approximation of a complex continuous C-space, to reason about connectivity. Checking collision-free connectivity in the roadmap requires expensive edge-evaluation computations, and thus, reducing the number of evaluations has become a significant research objective. However, in practice, we often face \emph{infeasible} problems:  those in which there is no collision-free path in the roadmap between the start and the goal locations. Existing studies often overlook the possibility of infeasibility, becoming highly inefficient by performing many edge evaluations. 

In this work, we address this oversight in scenarios where a prior roadmap is available; that is, the edges of the roadmap contain the probability of being a collision-free edge learned from past experience. To this end, we propose an algorithm called \emph{iterative path and cut finding} (\texttt{IPC}) that iteratively searches for a path and a cut in a prior roadmap to detect infeasibility while reducing expensive edge evaluations as much as possible. We further improve the efficiency of \texttt{IPC} by introducing a second algorithm, \emph{iterative decomposition and path and cut finding} (\texttt{IDPC}), that leverages the fact that cut-finding algorithms partition the roadmap into smaller subgraphs. We analyze the theoretical properties of \texttt{IPC} and \texttt{IDPC}, such as completeness and computational complexity, and evaluate their performance in terms of completion time and the number of edge evaluations in large-scale simulations.
\end{abstract}

\IEEEpeerreviewmaketitle

%===============================================================================

\section{Introduction}
\label{sec:intro}

Motion planning~\cite{lozano1979algorithm} is a crucial functionality for autonomous robots that enables them to move from one configuration to another without colliding with obstacles. The main objective of most motion planing algorithms is efficiency, that is, finding a solution as quickly as possible. 

Most existing algorithms start from the assumption that a successful motion plan exists, only terminating with failure after exhaustively searching the plan space or exceeding a time budget~\cite{basch2001disconnection,bretl2004multi,hauser2005learning,zhang2008efficient,mccarthy2012proving,rodriguez2019iteratively,li2020towards,varava2021free}.  However, in practical settings, it may not be uncommon for planning problems to be \emph{infeasible}.  For example, it may be impossible to grasp an object due to nearby objects blocking the way or the robot’s kinematic limitations.
% Many problems, however, may be \emph{infeasible} in practice, where a solution does not exist. For example, a target grasp pose cannot be reachable due to nearby objects occluding a target object or the kinematic limitation of a robot. 
The prospect of plan infeasibility is particularly salient in \emph{task and motion planning}~\cite{hauser2009integrating,garrett2021integrated}, where many motion planning subproblems are considered.
% This becomes a more demanding computational challenge in \emph{task and motion planning}~\cite{garrett2021integrated}, where many motion planning subproblems are involved. 
% In such cases, existing algorithms that do not reason over infeasibility can be highly inefficient; they generally terminate planning after exploring the entire search space or exhausting a time budget~\cite{rodriguez2019iteratively}. 
In this paper, we address the problem of detecting infeasibility (while still efficiently finding a solution when one exists), particularly in scenarios where learning from past experience is available.

% We consider the \emph{sampling-based motion planning} framework~\cite{choset2005principles,lavalle2006planning,elbanhawi2014sampling,mcmahon2022survey} as a motion planner. In sampling-based planning, a \emph{roadmap}~\cite{dobson2014sparse,coleman2015experience} (such as illustrated in Figure~\ref{fig:path_cut}), a graph that approximately captures connectivity in the configuration space (C-space~\cite{lozano1990spatial}), is a useful data structure for accumulating past planning results. 

\begin{figure}[t]
\centering
\includegraphics[width=0.35\textwidth]{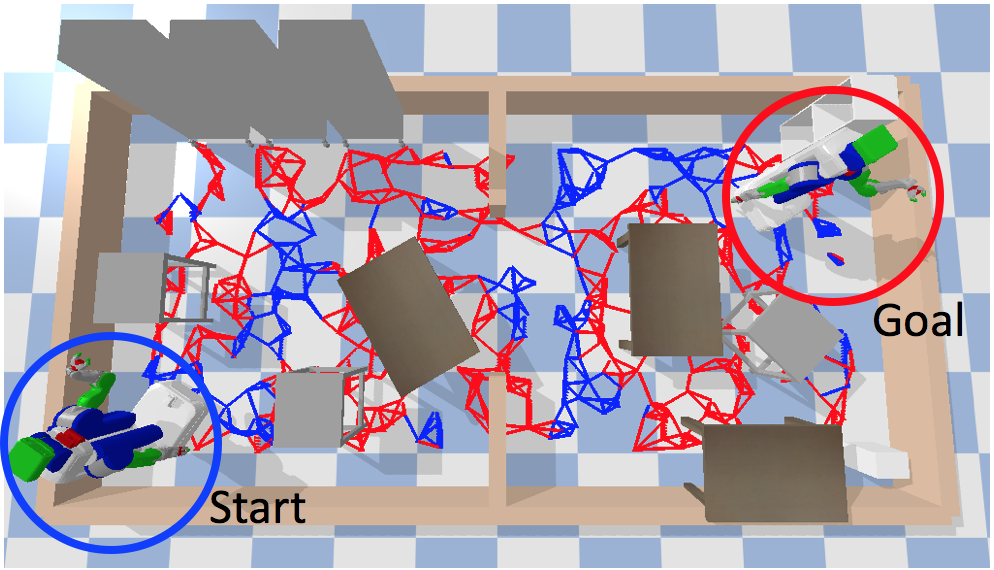}
\caption{Navigation task for a robot, where the blue and red circles represent the start and goal, respectively. A roadmap example of $500$ vertices is represented by red and blue lines. Red edges collide with an obstacle, whereas blue edges are collision-free.  An edge’s collision status (and thus its color) is initially unknown and can only be revealed by applying edge evaluation by a motion planner. Given a roadmap with prior probabilities over edges’ existences (based on past experience), the objective is to determine, with as few edge evaluations as possible, whether there is a collision-free path (a path on all blue edges) from the start to the goal.
}
\label{fig:navigation}
\end{figure}

Specifically, we focus on Step 2 of a three-step learning-based framework for obtaining motion plans (or determining their non-existence) from past experience. In \textbf{Step 1} of this overall framework, a representative \emph{roadmap}~\cite{dobson2014sparse,coleman2015experience} is learned from available training problems. This roadmap is a graph that approximately captures connectivity in the configuration space (C-space~\cite{lozano1990spatial}) as depicted in Figure~\ref{fig:navigation}. Roadmaps are widely used in motion planning, such as \emph{sampling-based motion planners}~\cite{choset2005principles,lavalle2006planning,elbanhawi2014sampling,mcmahon2022survey} which construct a roadmap to discretize a C-space and \emph{search-based planners}~\cite{cohen2010search} which generate a grid map corresponding to a roadmap. In \textbf{Step 2}, when presented with a query problem, we attempt to find a solution using the roadmap from Step 1, or alternatively to determine that no solution exists (\ie, infeasibility in the roadmap). In \textbf{Step 3}, if a solution has been found in Step 2, its quality is refined. On the other hand, if no solution was found, an attempt is made to prove infeasibility in the continuous C-space. Alternatively, points can be added to the roadmap from Step 1 to make it denser before repeating Step 2. In this work, we specifically concentrate on Step 2 by assuming the existence of a roadmap from Step 1, and we leave Step 3 for future work.

% To reason about infeasibility, we assume the existence of a \emph{roadmap}~\cite{dobson2014sparse,coleman2015experience}, a graph that approximately captures connectivity in the configuration space (C-space~\cite{lozano1990spatial}, as illustrated in Figure~\ref{fig:navigation}). Roadmaps are widely used in motion planning, such as \emph{sampling-based motion planners}~\cite{choset2005principles,lavalle2006planning,elbanhawi2014sampling,mcmahon2022survey} which construct a roadmap to discretize a C-space and \emph{search-based planners}~\cite{cohen2010search} which generate a grid map corresponding to a roadmap.

Roadmaps are a useful data structure for accumulating past planning results. Edges in a roadmap are hypothesized to be collision-free paths between their endpoints.  However, they may not actually be collision-free in the real world; in practice, they must be evaluated by a motion planner to check for collisions. We can exploit past experience to obtain the probability of local connectivity between vertices in a roadmap (\ie, their connecting edge being collision-free) and use the roadmap as prior knowledge to solve a query problem. 

One possible approach to obtaining connection probabilities is to use the \emph{counting principle}. Given training problems that may be either feasible or infeasible, whose C-spaces vary due to various numbers, shapes, and locations of obstacles, we generate a single representative roadmap by attempting to solve them with a given motion planner. Then, for each edge in the roadmap, we count the number of problems that the edge is confirmed collision-free out of the total number of problems; we apply the same procedure for all edges to form their existence probabilities. 
It is important to note that our method also works in cases where there is no prior knowledge, and in such cases, all edge connection probabilities are set to $0.5$ (see Section~\ref{subsec:comparison}).

Previous studies~\cite{choudhury2016pareto,narayanan2017heuristic,hou2020posterior} have leveraged prior roadmaps to find the shortest path between nodes. Their focus is on reducing the number of edge evaluations since edge evaluation is known to consume most of the computation time in motion planning due to its expensive collision-checking procedures~\cite{bohlin2000path,hauser2015lazy,bialkowski2016efficient,haghtalab2018provable}. However, their methods assume that a prior roadmap \emph{always} contains a solution, resulting in an ineffective approach to dealing with infeasible problems.
The search space of previous work is mostly a space of paths or edges; therefore, infeasible problems can only be determined after evaluating all possible paths or combinations of edges.

To detect infeasibility from a prior roadmap, we propose a different search strategy, namely \emph{iterative search over a path space and a cut space}. Our insight is that the search over both spaces can identify whether a given problem is feasible efficiently, whereas trying to find a path in an infeasible problem or trying to find a cut in a feasible problem requires exploring the entire search space.
% ; finding a ground-truth path from infeasible problems and a ground-truth cut from feasible problems requires exploring the entire search space. 

Furthermore, the result of a single path search execution can help guide the search for a cut and vice versa. We apply edge evaluations to edges found by the path search and then determine their ground-truth existence. Some edges may be identified as non-existing (\ie, colliding with obstacles), which can be leveraged by the cut search since a ground-truth cut \emph{must} contain one of those non-existing edges. Similarly, a ground-truth path \emph{must} include one of the existing edges identified by the cut search. 

Based on this observation, we design a \emph{complete} search algorithm that identifies feasibility of a given problem. We also propose a \emph{divide-and-conquer} algorithm that further improves efficiency by exploiting an abstract graph data structure over a roadmap. 

Our algorithms can complement existing methods by acting as preprocessors so that infeasible problems are identified as quickly as possible. If a given problem is confirmed feasible, we switch to existing approaches to find the shortest path.

In summary, we make three main contributions in this paper.
\begin{itemize}
\item We introduce two  probabilistic roadmap-based algorithms that perform path and cut searches to determine whether a query problem is feasible efficiently.
\item We analyze both algorithms’ completeness and computational complexity.
\item We empirically verify the performance of our methods through extensive simulations.
\end{itemize}

% Collision-checking queries occupy most of the computation time in sampling-based motion planning; thus, avoiding unnecessary queries as much as possible is desirable~\cite{hauser2015lazy}.

% Often, similar problems arise in practice~\cite{jurgenson2019harnessing}.

% To the best of our knowledge, all previous studies~\cite{choudhury2016pareto,narayanan2017heuristic,hou2020posterior,bhardwaj2021leveraging} assume that a prior roadmap \emph{always} contains at least one feasible path. 

\section{Problem Description}
\label{sec:prob}

Let $G=(V,E,p)$ be a weighted graph representing a prior roadmap learned from past experience. 
\begin{itemize}
\item $V$ is a vertex set. Each vertex $v\in V$ corresponds to a robot configuration in the C-space. 
% containing the configuration state denoted by $q(v)$.
\item $E$ is an edge set whose element is denoted by $e\in E$.
\item $p:E\rightarrow[0, 1]$ represents a Bernoulli probability of edge existence over all edges. Edge existence implies that the robot can traverse from one vertex to another without colliding with obstacles. A ground-truth edge existence can only be revealed by applying edge evaluation. $p(e)=1$ indicates that the edge is known with certainty to exist, whereas $p(e)=0$ indicates that the edge is known not to exist. $0<p(e)<1$ indicates that the edge needs to be evaluated to determine whether the path from vertex to vertex is collision-free.
% \item $c:E\rightarrow\{0,1\}$ represents a ground-truth existence that can be revealed after applying edge evaluation.
\end{itemize}

% We obtain weight values by the counting method from past problems. For example, if $e$ was in collision eight times out of $10$ problems, then $w(e)$ becomes $0.2$.

Start and goal configurations ($v_s$ and $v_g$) are given at query time, which can be connected to near vertices in $G$ if edge evaluation guarantees no collisions; this is a typical process for multi-query planners, such as PRM~\cite{kavraki1996probabilistic}. 
% Since $G$ does not include $v_s$, $v_g$, and corresponding new two edges, we must apply edge evaluation to check connectivity on the query problem. 
For those two new edges connecting $v_s$ and $v_g$ to $G$, $p=1$. For notational convenience, we treat the enlarged roadmap the same as $G$, including $v_s$ and $v_g$.

The objective of the feasibility detection problem is to determine whether a collision-free path connecting $v_s$ and $v_g$ exists in $G$ (and if so, to identify such a path) while minimizing the number of edge evaluations as much as possible. 
% Our objective is to minimize: $T+CN$, where $T$ is a wall-clock time, $C$ is a (large) constant factor (\ie, wall-clock time per edge evaluation), and $N$ is the number of evaluations. 

\section{Algorithms}
\label{sec:alg}

In this section, we propose an algorithm (\texttt{IPC}) that iteratively searches over the path and cut spaces to detect feasibility. We then present another algorithm (\texttt{IDPC}) based on \texttt{IPC} that effectively decomposes the search space to improve efficiency further.

\subsection{Iterative path and cut finding (\texttt{IPC})}
\label{subsec:ipc}

Since all we care about from $G$ is connectivity between a start and a goal vertex, we treat the edge-evaluation process as a black box computation. Any information on the configuration space in which $G$ is embedded is irrelevant to the algorithm design; our algorithms are agnostic to configuration values and their dimension. We thus focus only on the graph structure of $G$ in designing algorithms.

Our idea is to leverage existing off-the-shelf path-finding and cut-finding algorithms, both highly efficient due to the long history of their individual developments. A path-finding algorithm is used to certify connectivity in $G$, while a cut-finding algorithm confirms disconnectivity. We obtain the most probable path and cut as candidates from those algorithms and apply edge evaluations to check their ground-truth existence.\footnote{In the rest of the paper, we denote the output of a pathfinding or cut-finding algorithm as a \emph{candidate path or cut}, as their ground-truth existence has not yet been confirmed, and a ground-truth path or cut as simply a \emph{path or cut}.}

\begin{figure}[!htb]
\centering
\subfigure[Pathfinding result (path edges colored in blue).]{\includegraphics[width=0.45\columnwidth]{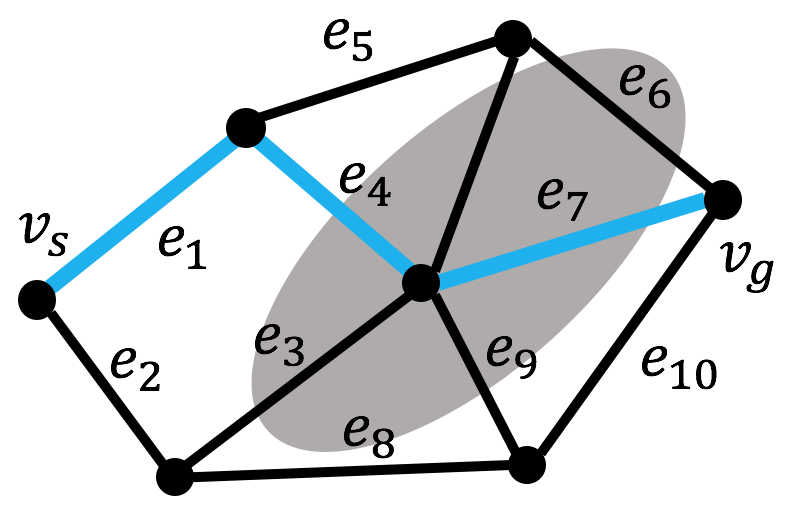}}
\subfigure[Cut-finding result (cut edges colored in red).]{\includegraphics[width=0.45\columnwidth]{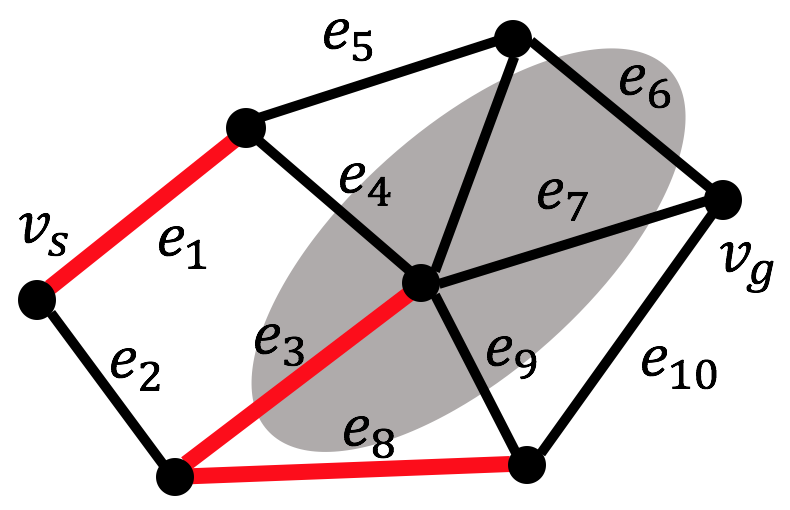}}
\caption{Examples of pathfinding and cut-finding executions. The gray shape represents an obstacle. In (a), \texttt{IPC} evaluates the existence of $e_1$, $e_4$, and $e_7$ and learns that a cut must contain (at least) one of $e_4$ and $e_7$ since $e_4$ and $e_7$ are found not to exist. In (b), \texttt{IPC} evaluates $e_1$, $e_3$, and $e_8$ and learns that a path must pass through (at least) one of $e_1$ and $e_8$ since $e_1$ and $e_8$ are collision-free.
} 
\label{fig:path_cut}
\end{figure}

As such, we propose an algorithm to determine the feasibility of a given problem through the path and cut search in $G$ without requiring too many edge evaluations, which we call \texttt{IPC}. \texttt{IPC} iteratively applies path-finding and cut-finding algorithms while the result of one informs the search for another in the next iteration and terminates when either a path or a cut is found. Figure~\ref{fig:path_cut} illustrates both scenarios.

Since $G$ is a positively weighted graph, we use Dijkstra's algorithm to find the most probable candidate path, a finite sequence of edges $P=(e_i)_{i=1}^{n}$. To find the most probable candidate cut, a finite edge set $C=\{e_j\}_{j=1}^{m}$,\footnote{Off-the-shelf minimum cut algorithms generally do not output a sequence of edges but an edge set instead, since as illustrated in Figure~\ref{fig:path_cut} (b), edges in a cut need not be adjacent.} we adopt the Push--relabel algorithm~\cite{cherkassky1997implementing} with its state-of-the-art efficiency. We present the time complexity of existing cut-finding algorithms in Table~\ref{table:complexity}.

\begin{table}[ht]
\begin{center}
\begin{tabular}{c|c}
\hline
Algorithms & Complexity \\
\hline
\hline
Ford--Fulkerson algorithm~\cite{cormen2022introduction} & $\mathcal{O}(|V|^2|E|)$ \\
Edmond--Karp algorithm~\cite{edmonds1972theoretical} & $\mathcal{O}(|V||E|^2)$ \\
Push--relabel algorithm~\cite{cherkassky1997implementing} & $\mathcal{O}(|V|^2\sqrt{|E|})$ \\
\hline
\end{tabular}
\caption{Complexities of minimum cut algorithms.}
\label{table:complexity}
\end{center}
\end{table}

We initialize edge values over $E$ using $p$.
Both Dijkstra’s and the Push--relabel algorithms reason over the \emph{sum} of edge values, whereas finding the most probable candidate path or cut requires reasoning over the \emph{product} of probabilities. To correct this mismatch, our algorithms reason over logarithmic $p$ values.
% Both algorithms need the addition operation over edge values to compute the most probable path and cut. To enable that operation with $p$, we define an edge value in the logarithmic form of $p$.
% Notice that we cannot directly use $p$ in $G$ for both algorithms as computing the most probable path or cut requires multiplication of $p$'s over edges. In contrast, both algorithms need the addition operation over edge values. 
Moreover, a path-finding algorithm seeks high $p$ values, while a cut-finding algorithm prefers low $p$ values. 
In the end, we augment $G$ with the \emph{weight} and \emph{capacity} used for finding a candidate path and cut, respectively.\footnote{Weight and capacity are terminologies for edge values used in the path-finding and cut-finding literature, respectively.} We denote the augmented graph by $\overline{G}=(V, E, p, p_w, p_c)$, where $p_w$ and $p_c$ are weight and capacity values over $E$ such that $p_w, p_c:E\rightarrow[0, \infty)$. For $0<p<1$, we compute $p_w$ and $p_c$ as follows.
\begin{equation*}
\begin{split}
p_w&=\log(1/p), \\
p_c&=\log(1/(1-p)).
\end{split}
\end{equation*}

Deterministic edges with $p$ values of $0$ or $1$ do not need evaluations. 
% We set $p_{\max}$ to either infinity or an arbitrarily large value depending on the path-finding and cut-finding software requirements.
% To handle those edges, we introduce a positive real-valued parameter $p_{\max}$ to ensure that edges with this value will not be chosen as a part of the candidate path and cut, satisfying that: 
% \begin{equation}
% p_{\max}>\max\big\{\forall P\max\{\sum_{i=1}^np_w(e_i)\}, \forall C\max\{\sum_{j=1}^mp_c(e_j)\}\big\}.
% \label{eqn:p_max}
% \end{equation}
To handle those edges, we set $p_w=\infty$ and $p_c=0$ when $p=0$ and $p_w=0$ and $p_c=\infty$ when $p=1$. 
The value of $\infty$ is used to ensure that edges with this value will not be chosen as a part of the candidate path or cut.
By doing so, we ensure that if a path exists, it must pass through (at least) one of the edges identified as existing by cut finding and that if a cut exists, it must contain (at least) one of the edges identified as non-existing by pathfinding.
% We experimentally obtain $p_{\max}$ from constructing $\overline{G}$ and present our choice in the experimental section.

\begin{algorithm}
\SetAlgoLined
\SetKwInOut{Input}{Input}
\SetKwInOut{Output}{Output}
\SetKwFunction{ExecutePathfinding}{ExecutePathfinding}
\SetKwFunction{ExecuteCutFinding}{ExecuteCutFinding}
\SetKwFunction{ChooseCutEdge}{ChooseCutEdge}
\SetKwFunction{ResetEdgeValues}{ResetEdgeValues}
\SetKwFunction{EvaluateEdgeExistence}{EvaluateEdgeExistence}
\Input{$\overline{G}=(V, E, p, p_w, p_c), v_s, v_g$}
\Output{$C$ or $P$}

% \textsf{Cut}, \textsf{Path}$\leftarrow\emptyset, \emptyset$

\While{True} {
$P\leftarrow$\ExecutePathfinding$(V, E, p_w, v_s, v_g)$

\If{\EvaluateEdgeExistence$(P)$}{
\Return{$P$} \tcp{A feasible problem.}
}

$p_c\leftarrow$\ChooseCutEdge$(P, p_c)$

$C\leftarrow$\ExecuteCutFinding$(V, E, p_c, v_s, v_g)$

\If{\EvaluateEdgeExistence$(C)$}{
\Return{$C$} \tcp{An infeasible problem.}
}

$p_c\leftarrow$\ResetEdgeValues$(P, p_c)$
}

\caption{\texttt{IPC}}
\label{alg:ipc}
\end{algorithm}

The pseudo-code of \texttt{IPC} is included in Algorithm~\ref{alg:ipc}. Dijkstra's algorithm (line $2$) and the Push--relabel algorithm (line $7$) are iteratively applied. 
In lines $3$ and $8$, \texttt{EvaluateEdgeExistence} performs edge evaluations on every edge in $P$ or $C$ and returns \textsc{true} if a path or cut is found.
% Edge evaluations are performed in lines $3$ and $8$. 
\texttt{EvaluateEdgeExistence} also updates $p_w$ and $p_c$ values depending on whether corresponding edges are in-collision (updating $p_w$ to $\infty$ and $p_c$ to $0$) or collision-free (updating $p_w$ to $0$ and $p_c$ to $\infty$). 

A candidate path $P$ found may contain multiple collision edges, and at least one of those edges must be a part of any cut (if a cut exists). Then, the cut-finding algorithm searches for a candidate cut that includes one of these edges.  We observe that choosing a center edge from the largest sequence of consecutive collision edges performs well in practice. We set $p_c$ for the chosen edge from $P$ to be $0$ while setting $p_c$ for the rest of the edges from $P$ to be $\infty$ (line 6), ensuring that a candidate cut must contain the chosen edge to disconnect $\overline{G}$. After confirming that a candidate cut $C$ found is not a cut, we reset the $p_c$ values of the other collision edges from $P$ back to $0$ for the next iteration (line 11). 
One may also apply a similar strategy to pathfinding by heuristically selecting one particular edge from $C$. However, it is more complex because $C$ consists of an unordered set of edges that are not necessarily adjacent (as can be seen in Figure~\ref{fig:path_cut} (b)). We leave full specification of such optimization to future work; pathfinding in Algorithm~\ref{alg:ipc} is applied over the entire graph.

\textbf{Completeness}: 
\texttt{IPC} is guaranteed to terminate by finding either a path or a cut.
\begin{theorem}
\label{thm:ipc}
\texttt{IPC} is complete.
\end{theorem}
\begin{proof}
Note again that pathfinding and cut-finding applications reveal the ground-truth existence of edges. Thus, if \texttt{IPC} can evaluate all edges in $\overline{G}$ before termination without missing some edges or resulting in an infinite loop, the feasibility of a given problem can be known. After edge evaluation, the values of $p_w$ and $p_c$ will become either $0$ or $\infty$. 
Those edges from $p_w$ and $p_c$ assigned to the value of $\infty$ will not be chosen by pathfinding or cut finding; both algorithms always explore new edges that have not yet been evaluated.
% Since $p_{\max}$ satisfies Equation~\ref{eqn:p_max}, pathfinding and cut finding do not output a path and cut from those edges having $p_{\max}$ values but always explore edges that are not evaluated yet. 
Since $\overline{G}$ has a finite number of edges, \texttt{IPC} will eventually evaluate all edges in the worst case. When the ground-truth existence of all edges is known, $\overline{G}$ must contain either a path or a cut. Thus, one of them is guaranteed to be found.
\end{proof}

The same guarantee also holds in the continuous C-space when $|V|\rightarrow\infty$, as a roadmap $\overline{G}$ is a discrete approximation. Counterparts of a path and cut in the continuous C-space are a one-dimensional curve and ($d-1$)-dimensional hyperplane when the ambient C-space is $d$-dimensional. If a connected curve from a start and goal exists, all hyperplanes meeting the curve must have at least one hole; otherwise, there cannot be a connected curve.

If $|V|$ is finite, the infeasibility proof provided by Theorem~\ref{thm:ipc} is valid only for the roadmap. Step 3 of the learning framework introduced in Section~\ref{sec:intro} aims to convert cut edges into ($d-1$)-dimensional hyperplanes to ensure infeasibility in the continuous C-space. The method~\cite{li2023sampling} can be invoked in this step to learn a separating hyperplane that disconnects the start and goal. 

\begin{figure}[!htb]
\centering
\includegraphics[width=0.40\textwidth]{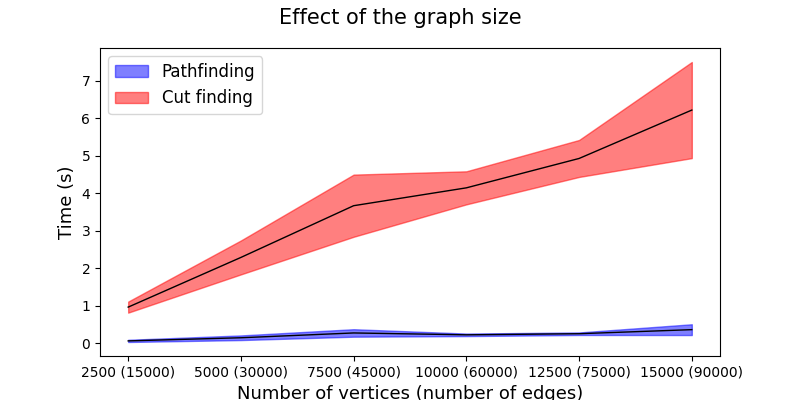}
\caption{A comparison of single instance execution time taken by Dijkstra's and the Push--relabel algorithms as a function of $|\overline{G}|$. The plot shows the mean and $95\%$ confidence interval from $10$ runs.
}
\label{fig:compute}
\end{figure}

\textbf{Complexity}: In \texttt{IPC}, the majority of computation occurs in finding a candidate cut (\ie, $\mathcal{O}(|V|^2\sqrt{|E|})$ of Push--relabel). Figure~\ref{fig:compute} shows the time comparison of Dijkstra's and the Push--relabel algorithms as the number of vertices and edges in $\overline{G}$ increases. We observe that the time taken for cut finding dominates that for pathfinding.
% increases exponentially while the increase in pathfinding is marginal. 

Motivated by this observation, we propose an improved algorithm over \texttt{IPC} in the next subsection, which decomposes $\overline{G}$ into smaller subgraphs so that the search space for cut finding becomes smaller. Since pathfinding computation is relatively cheap, we also present in Appendix~\ref{appen:more_pathfindings} the performance of \texttt{IPC} when increasing the number of pathfinding executions in each iteration.

\subsection{Iterative decomposition and path and cut finding (\texttt{IDPC})}
\label{subsec:idpc}

For the second algorithm, we exploit the fact that a candidate cut splits $\overline{G}$ into two separate \emph{induced subgraphs}\footnote{An induced subgraph is a special case of 
a subgraph, which satisfies not only that its vertices are a subset of vertices in $\overline{G}$ but also that
% for any vertices in the induced subgraph, 
it must contain all edges that exist in $\overline{G}$ whose both endpoint vertices exist in the induced subgraph.} that can be connected by some of the edges from the candidate cut if they are found to be collision-free (see Figure~\ref{fig:first_cut}) and that the two induced subgraphs can be obtained for free as a byproduct of executing the cut-finding algorithm. For convenience, we refer to induced subgraphs as subgraphs. We search for \emph{local} candidate cuts from individual subgraphs to determine their disconnectivity, further splitting them into even smaller subgraphs. Meanwhile, we aggregate this local information to determine \emph{global} disconnectivity between the start and goal. Consequently, the cut-finding algorithm runs on a smaller graph instead of $\overline{G}$.
We call this version iterative decomposition and path and cut finding (\texttt{IDPC}). Since \texttt{IDPC} iteratively decomposes $\overline{G}$ into multiple subgraphs and composes the results of local cut findings, it can be seen as a \emph{divide-and-conquer} approach.

\begin{figure}[!htb]
\centering
\includegraphics[width=0.35\textwidth]{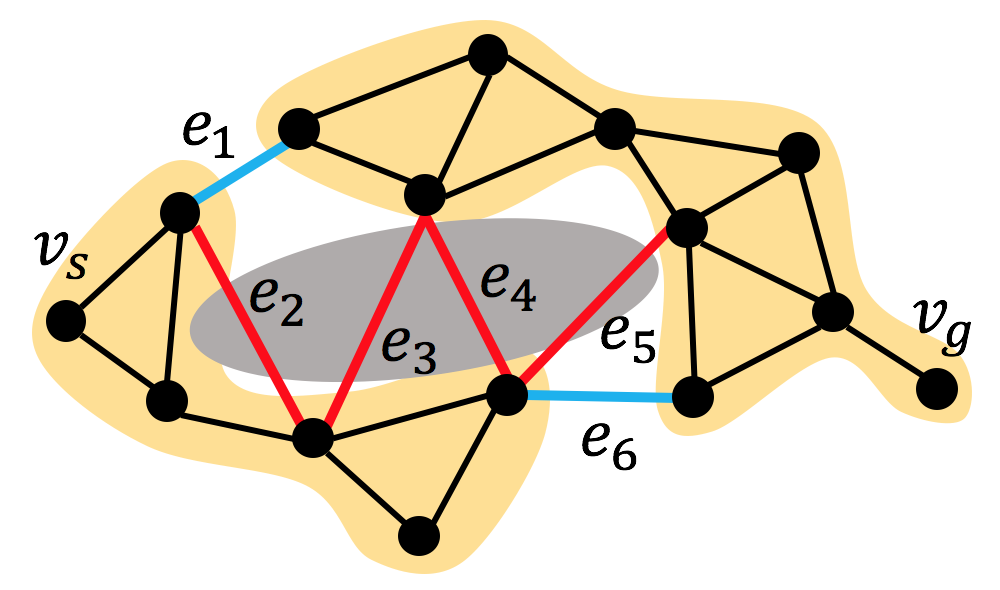}
\caption{
Description of a decomposed graph $\overline{G}$ and two separate induced subgraphs (marked by a yellow shape) generated by a candidate cut computed by the cut-finding algorithm. A candidate cut consists of $\{e_1,...,e_6\}$, where $e_1$ and $e_6$ are collision-free, and $\{e_2,...,e_5\}$ are collision edges as determined by evaluating them. As a result, the two induced subgraphs are connected by $e_1$ and $e_6$.
}
\label{fig:first_cut}
\end{figure}

\begin{figure}[!htb]
\centering
\includegraphics[width=0.35\textwidth]{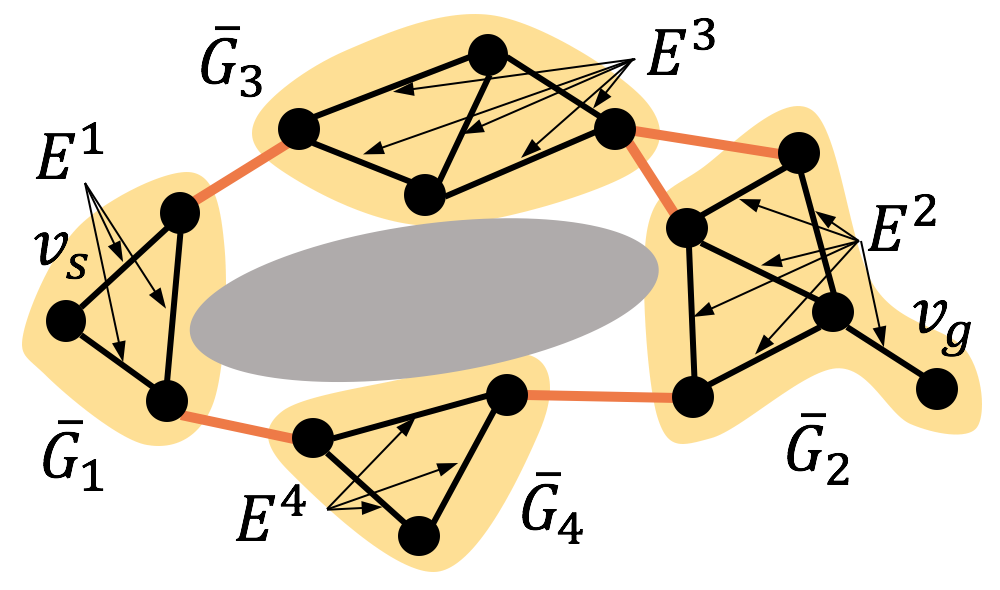}
\caption{Description of subgraphs generated by applying cut findings from the original graph. A yellow shape depicts each subgraph. Orange edges connecting neighboring subgraphs represent $\overline{C}$, confirmed as existing from a candidate cut after edge evaluations. Collision edges from a candidate cut will not be considered in further pathfindings or cut findings and, thus, are not drawn in this figure. In this example, $\{\overline{G}_k\}_{k=1}^4$ and orange edges together form $\overline{G}$, whose vertex set is $\cup_{k=1}^4V^k$.
}
\label{fig:subgraphs}
\end{figure}

In \texttt{IDPC}, the search space for a path is $\overline{G}$ whereas that for a cut is a set of connected components (\ie, subgraphs). We denote the set of subgraphs by $\{\overline{G}_k\}_{k=1}^g$, where $\overline{G}_k=(V^k, E^k, p^k, p_w^k, p_c^k)$. Subgraph edges satisfy the two conditions $\cup_{k=1}^g\{E^k\}=E\smallsetminus\cup C$ and $E^k\cap E^{k^\prime}=\oldemptyset$, as illustrated in Figure~\ref{fig:subgraphs}. The cut-finding algorithm identifies \emph{connecting edges} between subgraphs, and we only keep collision-free connecting edges in \texttt{IDPC}, denoted by $\overline{C}$. 
Notice that when $\overline{G}$ is split into $\overline{G}_1$ and $\overline{G}_2$ by $\overline{C}$, endpoint vertices of $\overline{C}$ in $\overline{G}_1$ form \emph{subgoals} for $v_s$, that is, any candidate paths must pass (at least) one of subgoals to reach $v_g$. Similarly, endpoint vertices of $\overline{C}$ in $\overline{G}_2$ form \emph{substarts} for $v_g$. As \texttt{IDPC} iteratively applies cut findings, any arbitrary subgraph $\overline{G}_k$ will contain substarts and subgoals unless completely disconnected from neighboring subgraphs.
% Endpoint vertices of $\overline{C}$ connecting between $\overline{G}_k$ and $\overline{G}_{k^\prime}$ define \emph{substarts} for $\overline{G}_k$ connected towards $v_g$ and \emph{subgoals} for $\overline{G}_{k^\prime}$ connected towards $v_s$. 
Due to the different search spaces for path and cut, the termination condition for declaring path existence is the same as in \texttt{IPC}, but we need another method for cut existence. 

\begin{figure}[!htb]
\centering
\includegraphics[width=0.35\textwidth]{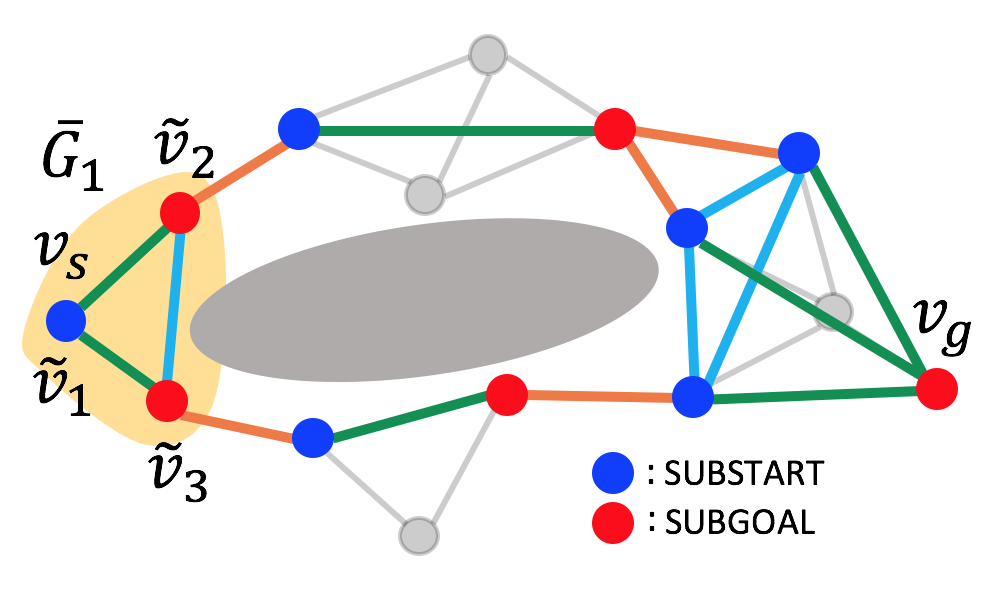}
\caption{Depiction of an abstract graph $\widetilde{G}$ generated from the example in Figure~\ref{fig:subgraphs}. Gray vertices and edges are not a part of $\widetilde{G}$ anymore. Blue vertices and red vertices correspond to substarts and subgoals, respectively. Green edges show the candidate connections between substarts and subgoals within the same subgraph. Orange edges represent $\overline{C}$, comprising the confirmed connections between subgraphs. Blue edges show the candidate connections among substarts or subgoals in the same subgraph. In $\overline{G}_1$, $\delta(\widetilde{v}_1)=v_s$ and $\Delta(\widetilde{v}_1\mbox{ or }\widetilde{v}_2\mbox{ or }\widetilde{v}_3)=1$. $\tau(\widetilde{v}_1)=\mbox{\textsc{substart}}$ and $\tau(\widetilde{v}_2\mbox{ or }\widetilde{v}_3)=\mbox{\textsc{subgoal}}$.
}
\label{fig:abstract_graph}
\end{figure}

For this purpose, we introduce an abstract graph $\widetilde{G}$, an undirected unweighted graph used for finding a cut in $\overline{G}$ from a set of cuts in $\{\overline{G}_k\}_{k=1}^g$. $\widetilde{G}$ ignores the detailed structure within $\overline{G}_k$ but captures the connectivity among $\{\overline{G}_k\}_{k=1}^g$ through $\overline{C}$ and the pairwise relationships among substarts and subgoals (Figure~\ref{fig:abstract_graph}). $\widetilde{G}$ must contain correspondence information concerning $\{\overline{G}_k\}_{k=1}^g$; thus, we define $\widetilde{G}$ as a tuple of $(\widetilde{V}, \widetilde{E}, c, \delta, \Delta, \tau)$ as follows.
\begin{itemize}
\item $\widetilde{V}$ is an abstract vertex set, where $\widetilde{v}\in\widetilde{V}$, corresponding to substarts and subgoals induced by $\overline{C}$ and $v_s$ and $v_g$ in $\overline{G}$. $\widetilde{V}$ satisfies that $\widetilde{V}\subseteq V$. In practice, $|\widetilde{V}|\ll|V|$.
\item $\widetilde{E}$ is an abstract edge set, where $\widetilde{e}\in\widetilde{E}$. Each $\widetilde{e}$ makes one of three types of connections: (1) \emph{candidate} connections between substarts and subgoals in the same $\overline{G}_k$; (2) \emph{confirmed} connections between substarts of $\overline{G}_k$ and subgoals of $\overline{G}_{k^\prime}$ enabled by $\overline{C}$; (3) \emph{candidate} connections among substarts or subgoals in the same $\overline{G}_k$. Candidate connections are those whose ground-truth connectivity has yet to be discovered. If any of the above candidate connections are identified by cut finding to be disconnected in $\overline{G}_k$, we do not maintain $\widetilde{e}$.
\item $c: \widetilde{E}\rightarrow \{\mbox{\textsc{true}}, \mbox{\textsc{false}}\}$ is a Boolean function, which maps $\widetilde{e}$ to \textsc{true} if a path between endpoint vertices of $\widetilde{e}$ has been identified in $\overline{G}_k$ or \textsc{false} otherwise, implying that a path may still exist.
\item $\delta: \widetilde{v}\rightarrow v$ maps $\widetilde{v}$ in $\widetilde{V}$ to $v$ in $V$ from which $\widetilde{v}$ is induced.
\item $\Delta: \widetilde{v}\rightarrow k$ maps $\widetilde{v}$ to an index $k$ of $\overline{G}_k$ to which $\delta(\widetilde{v})$ belongs.
\item $\tau: \widetilde{v}\rightarrow \{\mbox{\textsc{substart}}, \mbox{\textsc{subgoal}}\}$ is a Boolean function, which classifies the type of $\widetilde{v}$ into either \textsc{substart} or \textsc{subgoal}. 
$\exists \widetilde{v}=\inv\delta(v_s): \tau(\widetilde{v})=\mbox{\textsc{substart}}$, and $\exists \widetilde{v}=\inv\delta(v_g): \tau(\widetilde{v})=\mbox{\textsc{subgoal}}$.
\end{itemize}

Although $\widetilde{V}\subseteq V$, 
$\widetilde{G}$ is not a subgraph of $\overline{G}$ because $\widetilde{E}\not\subseteq E$. The reason for introducing the third type of abstract edge (\ie, blue edges in Figure~\ref{fig:abstract_graph}) is to cover the cases where the only feasible path visits the same subgraph multiple times by coming in and going out from a neighboring subgraph; without considering this case, \texttt{IDPC} is not complete. 

\begin{algorithm}
\SetAlgoLined
\SetKwInOut{Input}{Input}
\SetKwInOut{Output}{Output}
\SetKwFunction{ExecutePathfinding}{ExecutePathfinding}
\SetKwFunction{ExecuteCutFinding}{ExecuteCutFinding}
\SetKwFunction{ChooseCutEdge}{ChooseCutEdge}
\SetKwFunction{ResetEdgeValues}{ResetEdgeValues}
\SetKwFunction{EvaluateEdgeExistence}{EvaluateEdgeExistence}
\SetKwFunction{ReflectPathEvaluation}{ReflectPathEvaluation}
\SetKwFunction{InitializeAbstractGraph}{InitializeAbstractGraph}
\SetKwFunction{ChooseSubgraph}{ChooseSubgraph}
\SetKwFunction{ClusterSubstartsAndSubgoals}{ClusterSubstartsAndSubgoals}
\SetKwFunction{SubgraphPartition}{SubgraphPartition}
\SetKwFunction{CheckCutExistence}{CheckCutExistence}
\Input{$\overline{G}=(V, E, p, p_w, p_c), v_s, v_g$}
\Output{$C$ or $P$}

% \textsf{Cut}, \textsf{Path}$\leftarrow\emptyset, \emptyset$

$g=1$ \tcp{$\overline{G}_{k=1}=\overline{G}$.}

$\widetilde{G}\leftarrow$\InitializeAbstractGraph$(\overline{G}_{k=1})$

\While{True} {
$P\leftarrow$\ExecutePathfinding$(V, E, p_w, v_s, v_g)$

\If{\EvaluateEdgeExistence$(P)$}{
% \If{sum of weights in $P=\infty$} {
% \Return{$C$} \tcp{An infeasible problem.}
% }
% \Else{
\Return{$P$} \tcp{A feasible problem.
}
% } % Else
}

$\{\overline{G}_k\}_{k=1}^g, \widetilde{G}$, \texttt{subgraph\_ids}$\leftarrow$\newline\ReflectPathEvaluation$(\{\overline{G}_k\}_{k=1}^g, \widetilde{G}, P)$

$k^*\leftarrow$\ChooseSubgraph$(\{\overline{G}_k\}_{k=1}^g, P$, \texttt{subgraph\_ids}$)$

$p_c^{k^*}\leftarrow$\ChooseCutEdge$(P, p_c^{k^*})$

\texttt{substarts}, \texttt{subgoals}$\leftarrow$\newline\ClusterSubstartsAndSubgoals$(\overline{G}_{k^*}, \widetilde{G})$

$C_{k^*}\leftarrow$\ExecuteCutFinding$(V^{k^*}, E^{k^*}, p_c^{k^*}, v_s^{k^*}, v_g^{k^*}, \mbox{\texttt{substarts}, \texttt{subgoals}})$

$p_c^{k^*}\leftarrow$\ResetEdgeValues$(P, p_c^{k^*})$

$\{\overline{G}_{k^*}, \overline{G}_{k=g+1}\}, \widetilde{G}$\newline$\leftarrow$\SubgraphPartition$(\overline{G}_{k^*}, C_{k^*}, \widetilde{G})$

\If{\CheckCutExistence$(\widetilde{G})$}{
\Return{$C$} \tcp{An infeasible problem.}
}

}

\caption{\texttt{IDPC}}
\label{alg:idpc}
\end{algorithm}

With this new data structure, we now describe the pseudo-code of \texttt{IDPC} in Algorithm~\ref{alg:idpc}. The backbone of \texttt{IDPC} is similar to \texttt{IPC}, but it additionally includes $\{\overline{G}_k\}_{k=1}^g$ and $\widetilde{G}$ to reduce the search space for expensive cut finding. As initialization (lines 1 and 2), \texttt{IDPC} starts with a single subgraph equal to $\overline{G}$ and $\widetilde{G}$ consisting of two vertices (\ie, $\widetilde{v}_1=\inv\delta(v_s)$ and $\widetilde{v}_2=\inv\delta(v_g)$) and a single edge $\widetilde{e}_1$ connecting them. We also set $c(\widetilde{e}_1)=\mbox{\textsc{false}}$.

The process of searching for a candidate path $P$ is the same as \texttt{IPC} (lines 4-7). The result of edge evaluations over $P$ is used to update $\{\overline{G}_k\}_{k=1}^g$, assigning the values of $p_w^k$ and $p_c^k$ to either $0$ or $\infty$, depending on the collision status (line 8). If any portion of $P$ is a collision-free subpath (\ie, a collision-free path from a substart to a subgoal in a subgraph), \texttt{IDPC} sets the values of $c$ for the corresponding abstract edges in $\widetilde{G}$ to \textsc{true}. In line 9, \texttt{IDPC} chooses one subgraph $\overline{G}_{k^*}$ out of $\{\overline{G}_k\}_{k=1}^g$ to apply the cut-finding algorithm. In our experiments, we use a heuristic criterion to choose a subgraph that includes the largest number of collision edges in $P$. \texttt{IDPC} uses the same method as in \texttt{IPC} for selecting a particular edge used for a candidate cut within $\overline{G}_{k^*}$ (line 10).

\begin{figure}[!htb]
\centering
\includegraphics[width=0.35\textwidth]{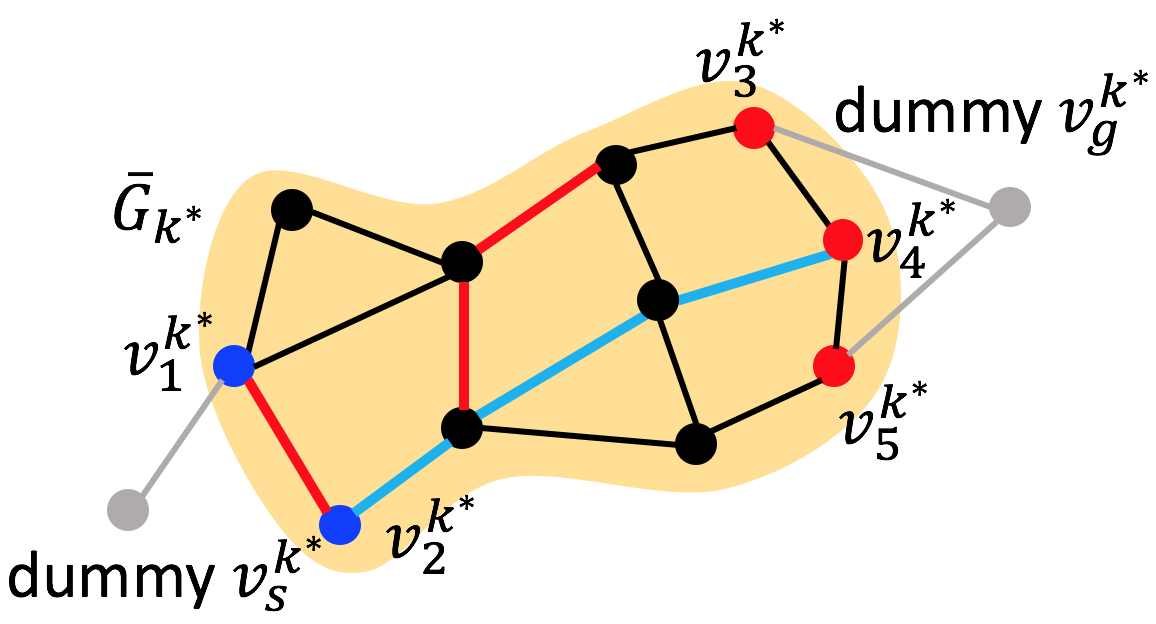}
\caption{Description of the clustering process for cut finding. Gray vertices ($v_s^{k^*}$ and $v_g^{k^*}$) and edges represent dummy. Blue vertices are substarts, and red vertices are subgoals. Note that $v_2^{k^*}$ and $v_4^{k^*}$ are not connected to dummy vertices because they are connected by a collision-free path colored in blue. The set of red edges is an example of a candidate cut that can disconnect $v_g^{k^*}$ from $v_s^{k^*}$ while not intersecting the blue path.
}
\label{fig:clustering}
\end{figure}

Unlike $\overline{G}$ which has a single start and a single goal, $\overline{G}_{k^*}$ may have multiple substarts and subgoals, making cut-finding algorithms inapplicable, since they can only accept one pair. Also, for the candidate connections between pairs of substart and subgoal (\ie, the first type of abstract edges), it is desirable if a single cut-finding execution identifies as many disconnections between pairs as possible rather than focusing on a single pair. Notice that the more disconnections between pairs are identified, the more balanced partition will likely be made in $\overline{G}_{k^*}$ compared to a partition obtained by trying to confirm the disconnection between a single pair. Neither method violates completeness, but the latter will likely incur more cut-finding executions, decreasing efficiency.
% Ideally, if all pairs are confirmed disconnected, we can safely disregard $\overline{G}_{k^*}$ from further cut-finding iterations unless a candidate path computed by pathfinding visits it.

To handle the inapplicability issue and to encourage balanced cuts, we cluster substarts and subgoals in $\overline{G}_{k^*}$ with dummy start and goal vertices ($v_s^{k^*}$ and $v_g^{k^*}$), as shown in Figure~\ref{fig:clustering}, and use dummy vertices as input to cut finding (line 11). By setting the values of $p_c^{k^*}$ of edges connecting to dummy vertices to $\infty$, we assure that a candidate cut can only be found within $\overline{G}_{k^*}$. One important point is that \texttt{IDPC} does not connect dummy vertices to a pair of substart and subgoal whose corresponding abstract edge in $\widetilde{G}$ satisfies that $c(\widetilde{e})=\mbox{\textsc{true}}$. In Figure~\ref{fig:clustering}, a collision-free path colored in blue from $v_2^{k^*}$ to $v_4^{k^*}$ shows such a case; no cuts can separate $v_4^{k^*}$ from $v_2^{k^*}$, and thus, we leave them out of candidate-cut consideration. \texttt{IDPC} searches for a candidate cut $C_{k^*}$ within $\overline{G}_{k^*}$ in line 12 and applies the same edge value reset process as in \texttt{IPC} in line 13. \texttt{IDPC} then removes dummy vertices and edges from $\overline{G}_{k^*}$.

\begin{figure}[!htb]
\centering
\subfigure[Before the subgraph partition.]{\includegraphics[width=0.49\columnwidth]{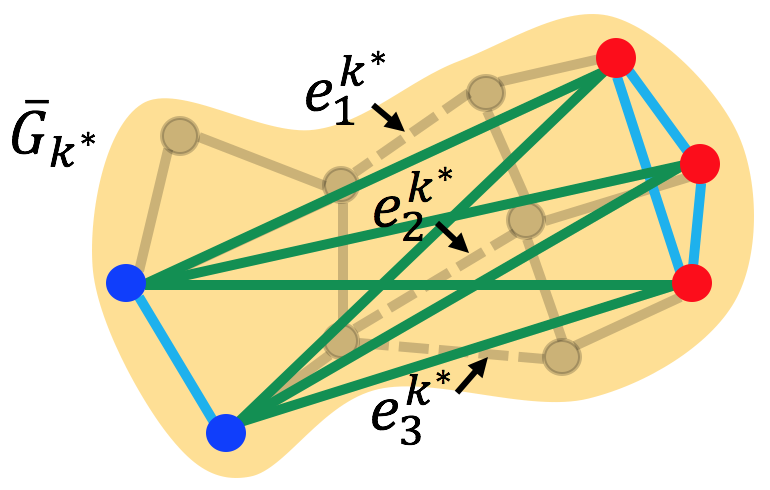}}
\subfigure[After the subgraph partition.]{\includegraphics[width=0.49\columnwidth]{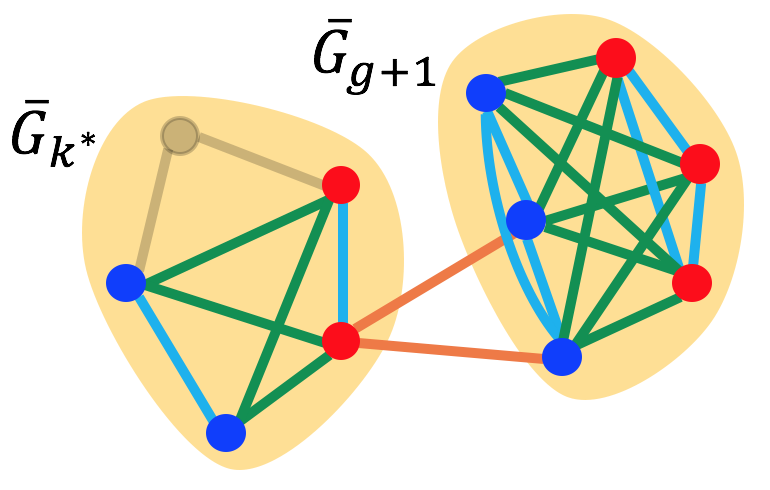}}
\caption{
Depiction of the subgraph partition process. The meaning of colors used for vertices and edges is the same as in Figure~\ref{fig:abstract_graph}. The candidate cut applied in this example is $C_{k^*}=\{e_1^{k^*}, e_2^{k^*}, e_3^{k^*}\}$ (\ie, dotted lines in (a)), where $e_1^{k^*}$ is in-collision while $e_2^{k^*}$ and $e_3^{k^*}$ are collision-free. 
} 
\label{fig:partition}
\end{figure}

In line 14, \texttt{IDPC} splits $\overline{G}_{k^*}$ into $\overline{G}_{k^*}$ and $\overline{G}_{g+1}$, followed by updating  $\widetilde{G}$ accordingly. Figure~\ref{fig:partition} illustrates the subgraph partition process, where a candidate cut found is $C_{k^*}=\{e_1^{k^*}, e_2^{k^*}, e_3^{k^*}\}$. After splitting $\overline{G}_{k^*}$ into two, \texttt{IDPC} applies edge evaluations to $C_{k^*}$ and leaves collision-free edges. To update $\widetilde{G}$, \texttt{IDPC} executes the following three steps. First, \texttt{IDPC} removes the first type of abstract edges (\ie, green edges in Figure~\ref{fig:partition} (a)) and some of the third type of abstract edges (\ie, blue edges in Figure~\ref{fig:partition} (a)) if $C_{k^*}$ splits endpoint vertices into different partitions. Second, \texttt{IDPC} adds new abstract vertices from collision-free edges that exist in $C_{k^*}$. Third, \texttt{IDPC} adds three types of new abstract edges (\ie, green and orange edges and blue edges among newly added abstract vertices in Figure~\ref{fig:partition} (b)). Function values of $\delta$, $\Delta$, and $\tau$ are assigned for the new abstract vertices and edges. 
For $c$, the second type of edges (\ie, orange edges in Figure~\ref{fig:partition}) is initialized with \textsc{true} as their connectivity has already been confirmed, but the other two types are initialized to  \textsc{false}.

Remember that the point of introducing $\widetilde{G}$ is to check for global disconnectivity from a set of cuts discovered from previous iterations (lines 15-17). Since $\widetilde{G}$ is an undirected unweighted graph, \texttt{IDPC} employs \emph{breadth-first search} to $\widetilde{G}$ to check the connectivity between abstract vertices corresponding to a start and a goal. Therefore, the termination condition for cut finding is the detection of disconnectivity in $\widetilde{G}$.

\texttt{IDPC} iterates the while loop until it finds either a path or a cut. Notice that all the computations regarding cut finding are now local. In Appendix~\ref{appen:pseudo_code}, we include the pseudo-code for the remaining functions in Algorithm~\ref{alg:idpc}. In Appendix~\ref{appen:idpc_example}, we show a pictorial example of how \texttt{IDPC} operates on a small toy roadmap.

\subsection{Analysis}
\label{subsec:analysis}

We analyze the completeness guarantee and time complexity of \texttt{IDPC}. The following theorem shows that the abstract graph $\widetilde{G}$ and subgraphs $\{\overline{G}_k\}_{k=1}^g$ still capture all possible paths and cuts by construction to preserve completeness. Therefore, like \texttt{IPC}, \texttt{IDPC} always finds a path if a given problem is feasible or a cut otherwise. 

\begin{theorem}
\label{thm:idpc}
\texttt{IDPC} is complete.
\end{theorem}
\begin{proof}
The proof is included in Appendix~\ref{appen:proof}.
\end{proof}

To analyze the time complexity of \texttt{IDPC}, we consider four variables: $V$, $E$, $\widetilde{V}$, and $\widetilde{E}$, where $V$ and $E$ are used for $\{\overline{G}_k\}_{k=1}^g$, whereas $\widetilde{V}$ and $\widetilde{E}$ are used for $\widetilde{G}$. Note that in the worst case, where every vertex in $\overline{G}$ forms an individual subgraph, $|V|\approx|\widetilde{V}|$ and $|E|\approx|\widetilde{E}|$. However, this case rarely occurs in practice, and generally, $|V|\gg|\widetilde{V}|$ and $|E|\gg|\widetilde{E}|$.

\begin{table}[ht]
\begin{center}
\begin{tabular}{c|c}
\hline
Components & Complexity \\
\hline
\hline
\texttt{ExecutePathfinding} (\ie, Dijkstra) & $\mathcal{O}(|E|+|V|\log|V|)$ \\
\texttt{ReflectPathEvaluation} & $\mathcal{O}(|\widetilde{V}||E|)$ \\
\texttt{ChooseSubgraph} & $\mathcal{O}(|E|)$ \\
\texttt{ClusterSubstartsAndSubgoals} & $\mathcal{O}(|\widetilde{V}|^2)$ \\
\texttt{ExecuteCutFinding} (\ie, Push--relabel) & $\mathcal{O}(|V|^2\sqrt{|E|})$ \\
\texttt{SubgraphPartition} & $\mathcal{O}(|\widetilde{V}||V|)$ \\
\hline
\end{tabular}
\caption{Complexities of major components in Algorithm~\ref{alg:idpc}.}
\label{table:idpc_complexity}
\end{center}
\end{table}

Table~\ref{table:idpc_complexity} shows the complexity of major computations in Algorithm~\ref{alg:idpc}, which can be derived from the pseudo-code in Appendix~\ref{appen:pseudo_code}. Like in \texttt{IPC}, the cut finding by Push--relabel dominates the overall computation. However, the cut-finding computation now applies to a subgraph whose size is $|V^k|$ and $|E^k|$, not the entire graph $\overline{G}$ as in the case of \texttt{IPC}, thus improving its efficiency. We achieve this improvement by embracing additional computations (as shown in Table~\ref{table:idpc_complexity}). To verify that those computations are relatively inconsequential, in Appendix~\ref{appen:analysis_idpc}, we analyze the  time taken by major components of \texttt{IDPC} as a function of graph size.

% \subsection{Variants of motion planners}

% Original PRM generates a tree graph; thus, if there exists a path between a start and a goal, it must be a unique path. Because of it, identifying feasibility requires a single path checking, which is trivial. 

\section{Evaluation}
\label{sec:eval}

\begin{figure*}
\centering
\subfigure{\includegraphics[width=0.55\columnwidth]{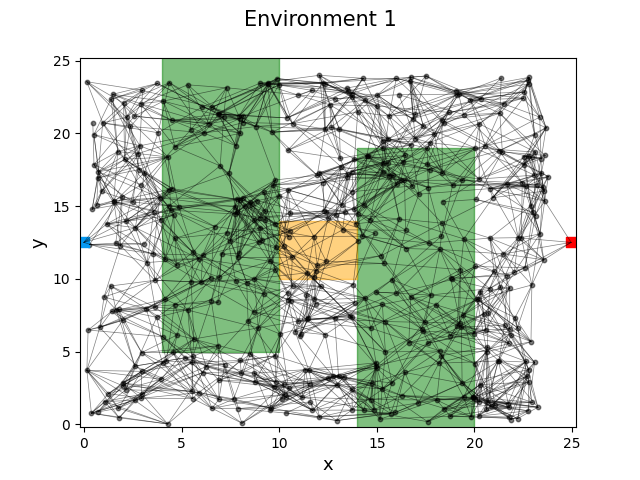}}
\hspace{-7mm}
\subfigure{\includegraphics[width=0.55\columnwidth]{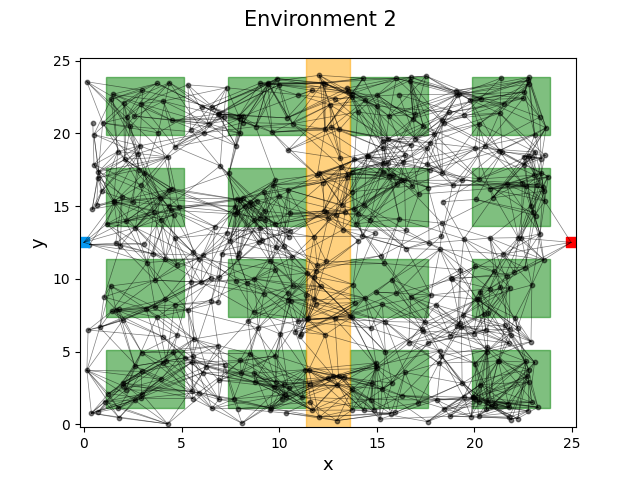}}
\hspace{-7mm}
\subfigure{\includegraphics[width=0.55\columnwidth]{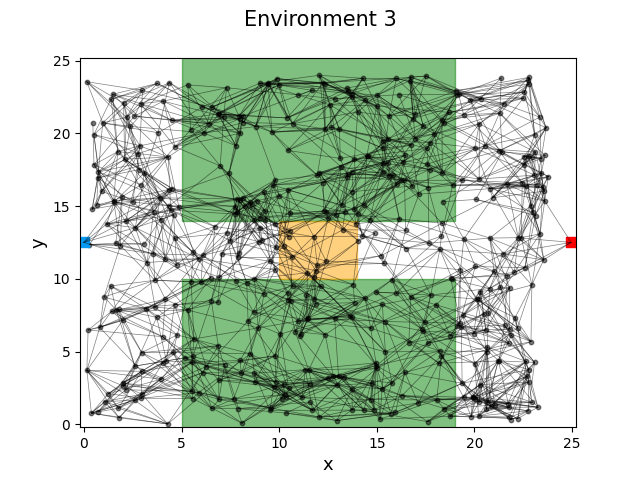}}
\hspace{-7mm}
\subfigure{\includegraphics[width=0.55\columnwidth]{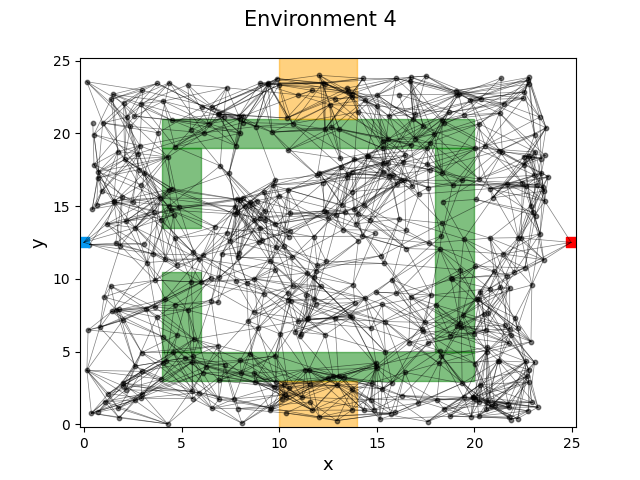}}
\caption{
Depiction of four environments used for comparison. Green and yellow polygons are obstacles. Infeasible and feasible problems are created with and without yellow obstacles, respectively. Blue squares are the start and red squares are the goal. Examples of a PRM roadmap consisting of $500$ vertices and $2000$ edges are shown in black.
} 
\label{fig:comparison_env}
\end{figure*}

% \begin{table}[ht]
% \begin{center}
% \begin{tabular}{c|c|c|c|c|c|c}
% \hline
% Graph size & \multicolumn{2}{c|}{Perfect} & \multicolumn{2}{c|}{Noisy} & \multicolumn{2}{c}{No prior} \\
% \hline
% Metrics & $\#$evaluations & time & $\#$evaluations & time & $\#$evaluations & time \\
% \hline
% \hline
% \texttt{IDPC} & & & & & \\
% \hline
% Pathfinding & & & & & \\
% \hline
% Cut finding & & & & & \\
% \hline
% BFS & & & & & \\
% \hline
% \end{tabular}
% \caption{}
% \label{table:comparison_size}
% \end{center}
% \end{table}

\begin{figure*}
\centering
\subfigure{\includegraphics[width=0.60\columnwidth]{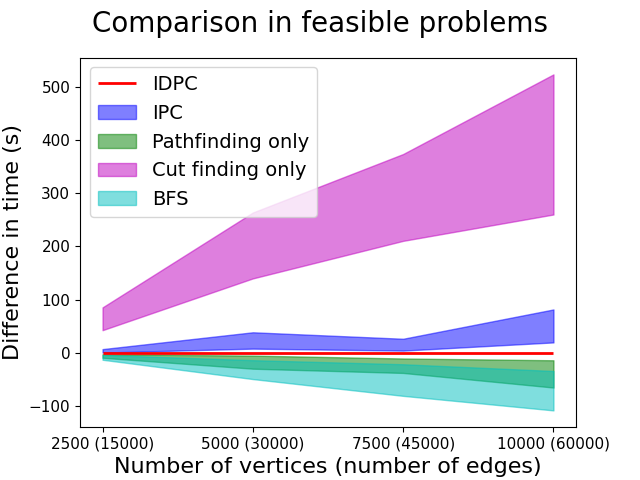}}
% \hspace{-7mm}
\subfigure{\includegraphics[width=0.60\columnwidth]{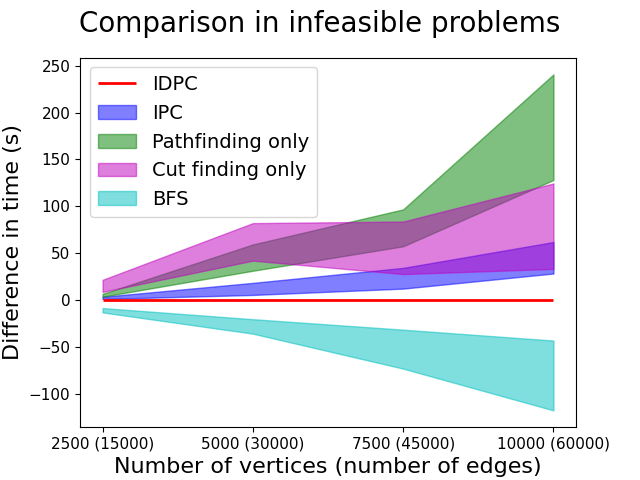}}
% \hspace{-7mm}
\subfigure{\includegraphics[width=0.60\columnwidth]{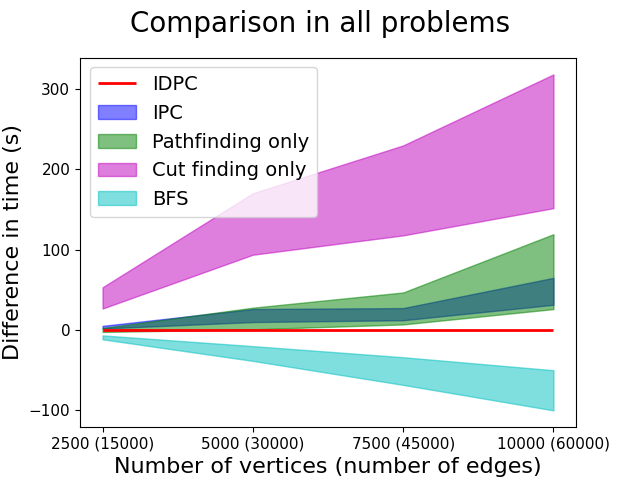}}
% \hspace{-7mm}
\subfigure{\includegraphics[width=0.60\columnwidth]{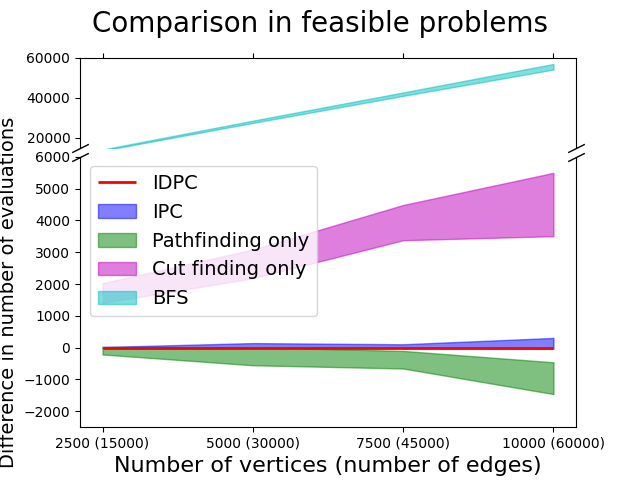}}
\subfigure{\includegraphics[width=0.60\columnwidth]{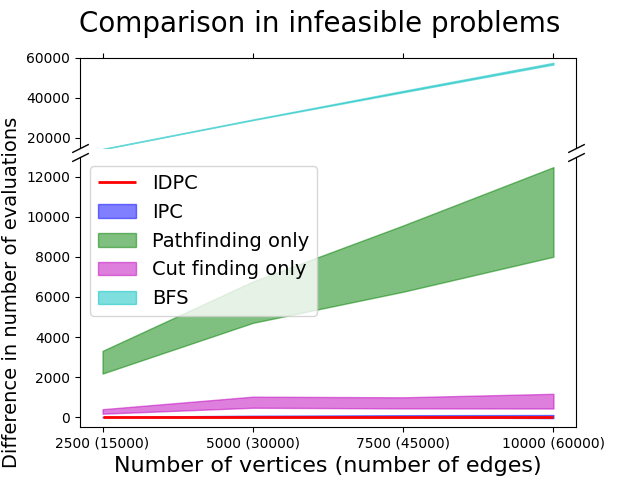}}
\subfigure{\includegraphics[width=0.60\columnwidth]{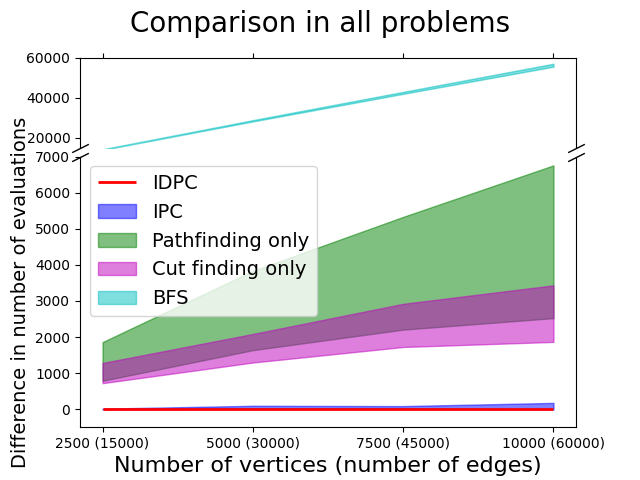}}
\caption{
Plots of the comparison analysis results. The top and bottom rows show the differences in completion time and the differences in the number of edge evaluations, respectively, compared to \texttt{IDPC}. The first, second, and third columns show feasible, infeasible, and a mixture of feasible and infeasible problem cases.
} 
\label{fig:comparison_result}
\end{figure*}

In this section, we validate our methods by the following set of evaluations. First, we show the comparison analysis against baselines in two-dimensional C space in terms of the completion time and the number of edge evaluations. Second, we conduct more realistic simulations where a C-space is high-dimensional. We report additional evaluations in the appendix, including an analysis of the computational overhead of \texttt{IDPC} compared to \texttt{IPC} (Appendix~\ref{appen:analysis_idpc}), how underlying graph topologies affect performance (Appendix~\ref{appen:graph_topology}), and the performance evaluation when increasing the number of pathfinding executions at each iteration (Appendix~\ref{appen:more_pathfindings}).

All experiments are conducted on an Intel Core i7-8665U CPU at 1.90 GHz with 16 GB of RAM. We adopt Dijkstra’s and the Push--relabel algorithms implemented in NetworkX~\cite{hagberg2008exploring}. All plots in this section show the mean and $95\%$ confidence interval obtained from multiple runs.

\subsection{Comparison with baselines}
\label{subsec:comparison}

We evaluate algorithms in terms of the number of edge evaluations and the time taken to detect feasibility. We do not include the edge evaluation time in the total computation time because it may differ depending on the local planning method, collision-checking algorithm used, robot shape, and how much approximation of the mesh shape is considered.

We compare against three baselines: (1) applying pathfinding only (\ie, Dijkstra's), (2) applying cut finding only (\ie, Push--relabel), and (3) the breadth-first search (BFS) based method. The first two baselines are used to show the consequence of neglecting to consider either the feasibility or infeasibility of a given problem. In particular, most existing methods in the literature that rely on the roadmap~\cite{choudhury2016pareto,narayanan2017heuristic,hou2020posterior} are represented by the first baseline since infeasibility is often overlooked. Other learning-based methods that are not based on the roadmap (referred to in Section~\ref{sec:related}) are omitted, as they do not fit into the proposed learning framework in Section~\ref{sec:intro} and typically do not provide infeasibility proofs.

BFS can serves as another baseline as it is guaranteed to visit all edges in a graph. However, because BFS does not reason about global disconnectivity as it searches in an edge space, and a roadmap is generally not a tree graph but contains many cycles, we need to modify BFS so that it can detect infeasibility. We create an additional graph containing only collision-free edges, incrementally constructed as BFS performs edge evaluations. After BFS (\ie, outer loop) reaches a goal, at every iteration of BFS, we apply another BFS (\ie, inner loop) to this new graph to determine whether a path exists from a start to a goal. If not, the outer-loop BFS continues, and we apply the same procedure. Infeasibility is declared if a path is not found after the outer-loop BFS search is exhausted.

Figure~\ref{fig:comparison_env} shows four domains we use for comparison. All domains are in two-dimensional C-space and yield both feasible and infeasible problems (\ie, without and with orange obstacles). We use PRM to generate a roadmap $\overline{G}$, but we also show the results of other roadmap types in Appendix~\ref{appen:graph_topology}.  We generate ten problems for feasible and infeasible scenarios, respectively, in each domain: ten roadmaps using different random seeds. For the values of $p$ in $\overline{G}$, we add random noise to the ground-truth values so that $p\sim U(0.3, 0.4)$ for collision edges and $p\sim U(0.6, 0.7)$ for collision-free edges, where $U$ represents a uniform distribution.

Figure~\ref{fig:comparison_result} shows the difference in the performance change compared to \texttt{IDPC} as the graph size increases. We set the performance of \texttt{IDPC} as a standard performance and gather statistics of the differences between all algorithms (\ie, \texttt{IPC} and baselines) and \texttt{IDPC}. In the plots of Figure~\ref{fig:comparison_result}, the results above the red line (\ie, \texttt{IDPC}'s performance) are worse than \texttt{IDPC}, whereas the results below the red line are better. Specifically, the results that do not overlap with each other (as well as with the red line) can be considered \emph{statistically significantly different}. We say that one method outperforms another when this is the case.

It can be seen that the pathfinding-only baseline outperforms both \texttt{IPC} and \texttt{IDPC} for feasible problems but performs worse when infeasible problems exist. The cut-finding-only baseline consistently performs worse than \texttt{IPC} and \texttt{IDPC} in all cases due to its heavy cut-finding computations, although its performance improves for infeasible problems. The BFS-based baseline has the shortest completion time but the largest number of edge evaluations among all methods. Like pathfinding only, this baseline exhibits weakness for infeasible problems; in the worst case, even if the neighboring edges of the goal comprise a cut, the BFS-based baseline still evaluates all edges.

\texttt{IDPC} outperforms \texttt{IPC} in terms of completion time in all cases although both require a similar number of edge evaluations. This result validates our motivation for proposing \texttt{IDPC}. Empirically, the number of evaluations required for \texttt{IDPC} is smaller than that for \texttt{IPC} marginally but not statistically significant. \texttt{IPC} still performs reasonably well compared to the baselines; although it overlaps with the pathfinding-only baseline regarding completion time, it outperforms pathfinding only by far in terms of the number of evaluations.

We conduct an additional evaluation where a different \emph{calibration} level of a prior roadmap is given. That is, we change how close the $p$ values in the prior roadmap are to the ground-truth values. In Figure~\ref{fig:calibration}, we show the results for three calibration levels of a roadmap when a mixture of feasible and infeasible problems is given: (1) \emph{perfect prior}, that is, $p=1$ for collision-free edges, and $p=0$ for collision edges; (2) \emph{noisy prior}, the same $p$ values used in Figure~\ref{fig:comparison_result}; and (3) \emph{no prior}, that is, all $p$ values are $0.5$.

In Figure~\ref{fig:calibration}, we observe that 
% the pathfinding-only baseline performs well for no prior but degrades for perfect prior and that the BFS-based baseline improves its performance as a prior roadmap becomes noisier. 
our methods degrade as a prior roadmap becomes noisier; a noisy prior incurs many unnecessary cut findings, increasing completion time. Alternatively, our methods clearly perform the best with a perfect prior. We further observe that \texttt{IDPC} shows strong \emph{robustness} to noise compared to \texttt{IPC}. We conjecture that the decomposition of a search space in \texttt{IDPC} helps avoid searching in unnecessary regions of the space, which is a particular strength of \texttt{IDPC}.

\begin{figure}[H]
\centering
\subfigure{\includegraphics[width=0.51\columnwidth]{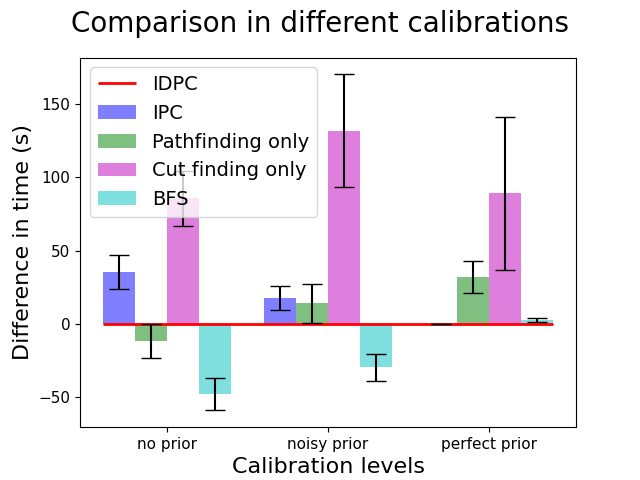}}
\hspace{-5mm}
\subfigure{\includegraphics[width=0.51\columnwidth]{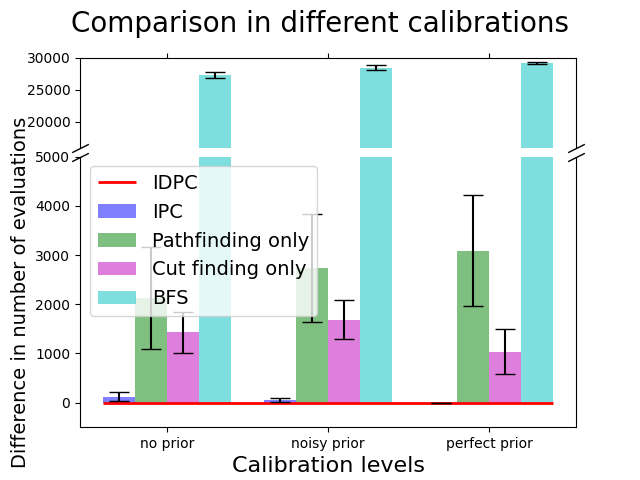}}
\caption{
Plots comparing the results for three calibration levels of a prior roadmap. The tip of the bar graphs represents the mean, and the error bars represent $95\%$ confidence intervals.
} 
\label{fig:calibration}
\end{figure}

\subsection{Performance on higher-dimensional C-space}
\label{subsec:high_cspace}

To this point, all experiments have been conducted in two-dimensional C-spaces for ease of visualization. Here, we investigate more realistic scenarios where a C-space is high-dimensional: the navigation task in Figure~\ref{fig:navigation}, having a $3$-dimensional C-space, and the manipulation task in Figure~\ref{fig:manipulation}, having a $7$-dimensional C-space. We obtain a prior roadmap as follows. We inject random Gaussian noise into the location of obstacles to generate multiple problem instances. We run PRM to get a roadmap and compute the edge-existence probability for all edges from the generated problem instances. We then selectively create ten feasible and ten infeasible new problems as query problems. 

\begin{figure}[!htb]
\centering
\includegraphics[width=0.30\textwidth]{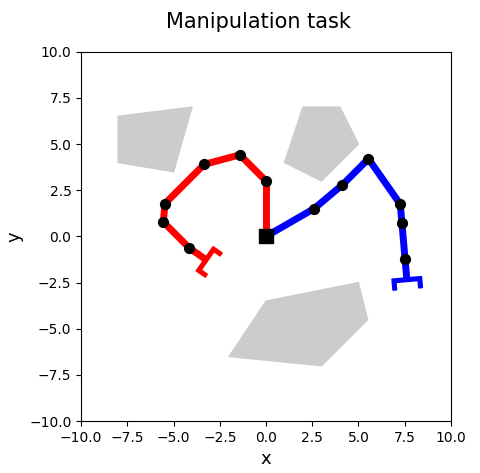}
\caption{Description of the manipulation task. The blue and red manipulators represent the start and the goal, respectively. The gray polygons are obstacles. 
}
\label{fig:manipulation}
\end{figure}

\begin{figure}[!htb]
\centering
\subfigure{\includegraphics[width=0.51\columnwidth]{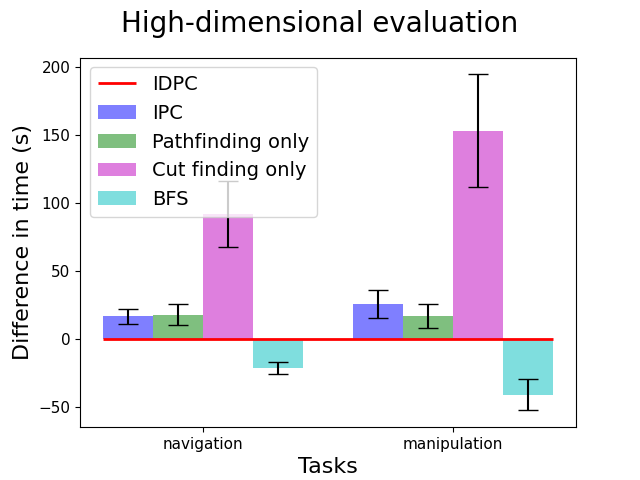}}
\hspace{-5mm}
\subfigure{\includegraphics[width=0.51\columnwidth]{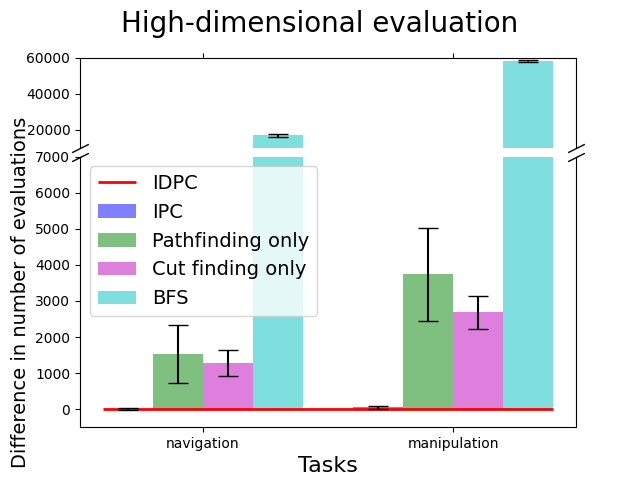}}
\caption{
Plots comparing the results for navigation and manipulation tasks. For the navigation task, PRM generates a roadmap consisting of $5000$ vertices with $20000$ edges. The roadmap includes $10000$ vertices with $60000$ edges for the manipulation task.
} 
\label{fig:high_dimension}
\end{figure}

Figure~\ref{fig:high_dimension} shows the results for both navigation and manipulation tasks. A similar trend in Section~\ref{subsec:comparison} can be seen here (see the third column of Figure~\ref{fig:comparison_result}). As previously, \texttt{IDPC} performs the best in all cases other than the BFS-based baseline's completion time. As explained in Section~\ref{subsec:ipc}, this result indicates that our algorithms' performance is robust to the dimensionality of the C-space and is only affected by the structure and size of the prior roadmap and the ground-truth existence of its edges. 

% Ten roadmaps $G$ are randomly generated for comparison study, where five roadmaps are connected, and the rest are disconnected. Each $G$ consists of $5,000$ nodes and $359,368$ edges. $p$ values of $E$ are obtained from the ground-truth existence with a small added noise. We choose the Push--relabel algorithm for finding the minimum cut.

% \begin{figure}[ht]
% \centering
% \includegraphics[width=0.50\textwidth]{sim_roadmap.png}
% \caption{One example of the environments in simulation. Blue edges are collision-free edges, while red edges are collision edges.}
% \end{figure}

% Roadmap characterization: calibration, maximum degree (related to the radius parameter), number of nodes.

% - Evaluation on different number of path finding trials when fixing the number of cut finding trial, which is one.

% - Evaluation on different dimensionality of the C-space.

% - Evaluation on different strategy of choosing a particular cut edge from a path found (e.g. instead of choosing the medium one). Sufficient to do that in IPC only.

\section{Related Work}
\label{sec:related}

\textbf{Infeasibility detection in motion planning}: Sampling-based planners are \emph{semi-decidable}, meaning that they will eventually find a feasible path when one exists but does not know how to terminate when no path exists~\cite{yap2013soft}. Several methods have been proposed in the literature to deal with this issue, which can be classified into three approaches: (1) direct infeasibility detection, (2) designing a stopping strategy, and (3) learning to predict infeasibility. 

The approach of direct infeasibility detection focuses on detecting whether a goal is disconnected from a start rather than finding a path. The work~\cite{basch2001disconnection,bretl2004multi,li2020towards} proposes disconnection proofs and applies their method to a query problem to check whether their disconnection proof holds; if it holds, the problem is guaranteed infeasible. The drawback is that the cost of computing their proofs is prohibitive as complex optimizations are involved.
The work~\cite{zhang2008efficient,mccarthy2012proving,varava2021free} also addresses proving path non-existence by introducing approximation to a C-space, such as cell decomposition~\cite{zhang2008efficient,mccarthy2012proving} and a finite set of slices~\cite{varava2021free}.

The stopping strategy aims at detecting infeasibility while finding a feasible path. Instead of terminating a planner when exceeding a predetermined time budget, this approach adds computation to the main pathfinding algorithm to actively stop when a particular criterion is met. The work~\cite{hadfield2016sequential} studies optimization-based motion planning, proposing to evaluate the improvement in the objective value to terminate the planner.
The sparse roadmap~\cite{simeon2000visibility,dobson2014sparse,coleman2015experience,orthey2021sparse}, which we use as a baseline in Appendix~\ref{appen:graph_topology}, is guaranteed to terminate for infeasible problems and outputs probabilistic infeasibility proofs.

Learning-based methods assume access to past experience to learn feasibility classifiers, owing to machine learning techniques.
The work~\cite{hauser2005learning,wells2019learning} uses support vector machines (SVM) to learn a feasibility classifier for multi-step planning, such as task and motion planning.
The work~\cite{li2023sampling} learns a separating manifold between a start and a goal using radial basis function kernel SVM. 
To deal with image inputs, the work~\cite{driess2020deep,bouhsain2023learning,xu2022accelerating,park2022scalable} designs convolutional neural network-based feasibility classifiers.

Our work can be seen as a mixture of all three approaches in the sense that cut finding aims at direct infeasibility detection, that our methods stop earlier than pathfinding only or cut finding only, and that a prior roadmap is learned from past experience.

\textbf{Efficient collision checking}: Edge evaluations that require many collision checks as a subroutine are empirically known to consume the most computation in motion planning~\cite{choudhury2016pareto,narayanan2017heuristic,hou2020posterior,bhardwaj2021leveraging}. Thus, efficient collision-checking strategies have been proposed in previous studies. 

The work~\cite{bialkowski2013free,bialkowski2016efficient} proposes guiding the sampling distribution at each iteration using information learned from previous iterations to reduce collision checks. The work~\cite{murray2016robot} designs collision-detection circuits that can run three orders of magnitude faster than existing algorithms by exploiting parallelism. The work~\cite{pan2016fast} develops efficient hashing techniques to predict the collision probability of a query sample using stored collision results from previous queries. The work~\cite{leven2002framework,bekris2007greedy,otte2016rrtx} investigates efficient collision checking in scenarios with dynamically changing environments.

In motion planning, the notion of \emph{laziness}~\cite{bohlin2000path,hauser2015lazy,dellin2016unifying,haghtalab2018provable} is introduced to defer expensive collision-checking procedures and apply them only when necessary. There exists a line of research on lazy planning in the literature, such as Lazy PRM~\cite{bohlin2000path}, Lazy PRM*~\cite{hauser2015lazy}, and Lazy SP~\cite{dellin2016unifying}. 

Recently, \emph{neural} approaches for collision detection methods~\cite{das2020learning,danielczuk2021object,kew2021neural,zhi2022diffco} or edge evaluation~\cite{yu2021reducing} have been proposed leveraging the rapid advances of deep learning. These methods require collecting data and training their model, but they have been shown to perform query problems more efficiently than methods without learning. 

Our work also addresses how to obtain efficient collision checking but on a specific data structure, a probabilistic prior roadmap.

\textbf{Learning framework using a prior roadmap}: We are not the first to explore a probabilistic prior roadmap in motion planning. 
Fuzzy PRM~\cite{nielsen2000two} first proposes incorporating edge-existence probability on the edges of a roadmap. 
The work~\cite{choudhury2016pareto,narayanan2017heuristic,hou2020posterior} also employs a prior roadmap, whose objective is to find the shortest path with the minimum possible edge evaluations. POMP~\cite{choudhury2016pareto} solves a Pareto-optimality problem that trades-off between path length and collision measure. AEE*~\cite{narayanan2017heuristic} is an anytime algorithm formulated in the Markov decision process, aiming for optimal path length in expectation. PSMP~\cite{hou2020posterior} formulates a regret minimization problem in Bayesian reinforcement learning.
However, the possibility of infeasibility is ignored in those studies, which is the main motivation for this paper.

\section{Conclusion}
\label{sec:conclusion}

In this work, we address the problem of efficiently determining whether a query problem on a probabilistic prior roadmap is feasible. To this end, we propose two complete algorithms (\ie, \texttt{IPC} and \texttt{IDPC}) that are guaranteed to find either a path or a cut, outperforming the baseline methods.

We observe that the performance of our methods is upper limited by the choice of a cut-finding algorithm. Although Push--relabel is the most efficient existing minimum cut algorithm, it is still expensive for large roadmaps since it only searches for the minimum cut. Alternatively, instead of the minimum cut, a suboptimal cut may suffice for our applications; it may be easier to compute but still perform well in practice because we deal with a probabilistic prior roadmap. There are some approximation schemes to cut finding in the literature~\cite{leighton1989approximate,leighton1999multicommodity}. We leave investigating the applicability of our ideas to those approximation algorithms to improve efficiency further as future work.

Also, incorporating our methods in multi-step planning, such as task and motion planning, would be another promising direction, especially in cases where checking for infeasibility is a crucial bottleneck.

\section*{Acknowledgments}
This work has taken place in the Learning Agents Research Group (LARG) at UT Austin.  LARG research is supported in part by NSF (FAIN-2019844), ONR (N00014-18-2243), ARO (W911NF-19-2-0333), DARPA, Bosch, and UT Austin's Good Systems grand challenge. Peter Stone serves as the Executive Director of Sony AI America and receives financial compensation for this work. The terms of this arrangement have been reviewed and approved by the University of Texas at Austin in accordance with its policy on objectivity in research.

\bibliographystyle{ieeetr}
\bibliography{fd_refs}

\clearpage
\appendix
\input{appendix}

\end{document}

%% file: appendix.tex
\subsection{Proof of Theorem~\ref{thm:idpc}}
\label{appen:proof}

We first present the following lemma, which will be used subsequently.

\begin{lemma}
\label{lemma:abs_edge}
Abstract edges in $\widetilde{G}$ consider all possible connectivity in $\overline{G}$.
\end{lemma}
\begin{proof}
Three types of abstract edges in $\widetilde{G}$ consider all possible paths in each subgraph, thereby, all possible connectivity in $\overline{G}$. The first type (\ie, green edges in Figure~\ref{fig:abstract_graph}) covers all possible paths between different neighboring subgraphs. The second type (\ie, orange edges in Figure~\ref{fig:abstract_graph}) covers all collision-free edges connected to neighboring subgraphs. The third type (\ie, blue edges in Figure~\ref{fig:abstract_graph}) covers all possible paths from and to the same neighboring subgraph. Since $\overline{G}$ consists of a collection of subgraphs, $\widetilde{E}$ of $\widetilde{G}$ completely covers all possible connectivity in $\overline{G}$.
\end{proof}

For \texttt{IDPC} to be complete, it must output a path if a given problem is feasible or a cut otherwise. As \texttt{IDPC} searches for a path from the entire roadmap graph $\overline{G}$, completeness for path existence follows \texttt{IPC}. Thus, we must show that introducing an abstract graph $\widetilde{G}$ and subgraphs $\{\overline{G}_k\}_{k=1}^g$ does not violate completeness for cut existence. Since \texttt{IDPC} determines infeasibility by checking disconnectivity from the start to the goal in $\widetilde{G}$, the following lemma proves the cut existence part for Theorem~\ref{thm:idpc}.

\begin{lemma}
\label{lemma:cut}
A cut exists in $\overline{G}$ if and only if $\widetilde{G}$ is disconnected from the start to the goal.
\end{lemma}
\begin{proof}
$(\Leftarrow)$ Let's assume that a path exists in $\overline{G}$ if $\widetilde{G}$ is disconnected from the start to the goal. This statement with Lemma~\ref{lemma:abs_edge} implies that edge evaluations from cut findings identify some edges as non-existing, although they are collision-free. Since the output of the edge-evaluation process is always correct, this statement leads to a contradiction. Thus, the sufficient statement (\ie, $\Leftarrow$) holds. 

$(\Rightarrow)$ The necessary statement involves more because we must show that $\widetilde{G}$ allows to evaluate all possible cuts in $\overline{G}$ before termination if a cut exists. To show this, we present the correctness of \texttt{IDPC} (\ie, not precluding any possible cuts). Three subgraph manipulations are related to cut finding in \texttt{IDPC}, so we show their correctness individually.

First, a subgraph whose substarts and subgoals are disconnected from neighboring subgraphs is not considered for further cut finding. This does not affect the correctness because a path going through this subgraph cannot exist; we can safely remove this type of subgraphs from consideration.

Second, \texttt{IDPC} chooses a particular subgraph out of candidate subgraphs for cut finding if the number of consecutive collision edges is the largest. This choice covers all candidate subgraphs in the worst case because pathfinding precludes each candidate sequentially.

Third, our clustering strategy for cut finding precludes a pair of substart and subgoal whose $c$ value is \textsc{true} (\ie, a collision-free path between substart and subgoal is found) from dummy vertices. Once a pair of substart and subgoal is found to have a collision-free path, further cut-finding executions cannot split this pair into different subgraphs. This invariance allows the removal of those connected pairs safely. Further pathfinding executions will identify paths in a subgraph, and the corresponding pairs will not be considered. Thus, while respecting Lemma~\ref{lemma:abs_edge}, the clustering method always considers all possible valid cuts in a subgraph.

\end{proof}

\subsection{Analysis on \texttt{IDPC}}
\label{appen:analysis_idpc}

In Appendices~\ref{appen:analysis_idpc},~\ref{appen:graph_topology}, and~\ref{appen:more_pathfindings}, we use the setting of $5000$ vertices with $30000$ edges and the same calibration level used in the leftmost figure in Figure~\ref{fig:comparison_env} as a prior roadmap. We also generate the same number of feasible and infeasible problems.

As additional components are introduced in \texttt{IDPC}, we measure the time each takes for a single iteration of Algorithm~\ref{alg:idpc} and empirically evaluate the complexities shown in Table~\ref{table:idpc_complexity}. 

\begin{figure}[!htb]
\centering
\includegraphics[width=0.45\textwidth]{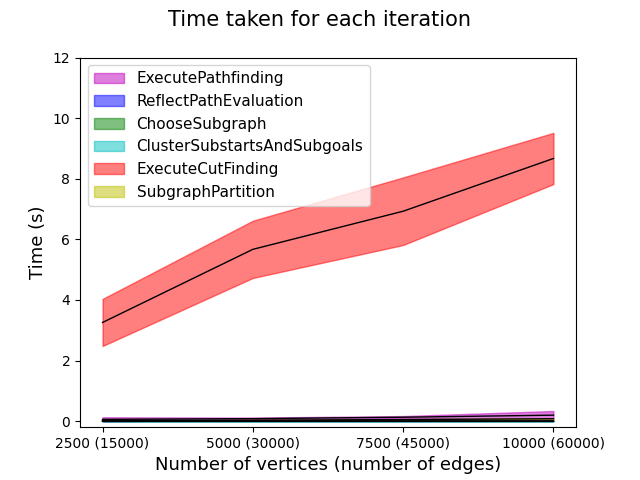}
\caption{A time comparison among major components in Table~\ref{table:idpc_complexity}.
}
\label{fig:analysis_idpc}
\end{figure}

Figure~\ref{fig:analysis_idpc} presents the single iteration execution time obtained by running ten instances when increasing a graph size. It can be seen that all the computations but cut finding are relatively marginal and scale well with the graph size.
% supporting the use of $|\widetilde{V}|$ and $|\widetilde{E}|$ in the complexity analysis.
The result empirically validates the complexity analysis in Table~\ref{table:idpc_complexity} and supports the use of $|\widetilde{V}|$ and $|\widetilde{E}|$.

\subsection{Effect of graph topologies}
\label{appen:graph_topology}

Different motion planners generally generate fundamentally different topologies of a roadmap. In this evaluation, we study how a roadmap's topology affects our methods' performance. Besides PRM, we additionally implement a grid map used in search-based planning and SPARS~\cite{dobson2014sparse}, a sparse roadmap with a provable suboptimality guarantee, used in sampling-based motion planning. We omit the subroutine of the SPARS implementation aimed for suboptimality because this computation is highly costly and does not affect our analysis critically. Examples of roadmaps generated can be seen in Figure~\ref{fig:topologies_env}.

\begin{figure}[!htb]
\centering
\subfigure{\includegraphics[width=0.51\columnwidth]{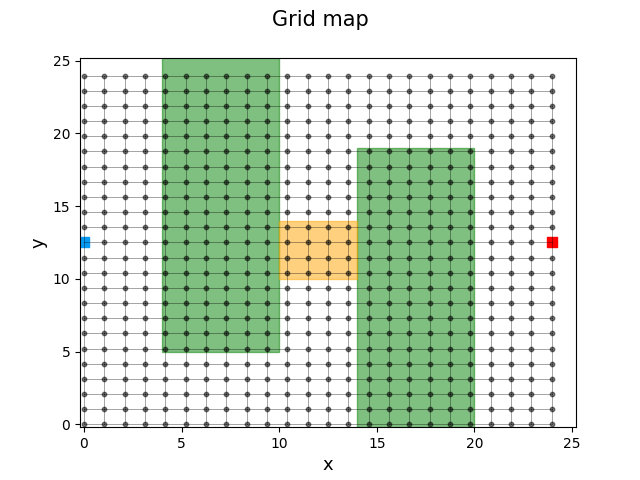}}
\hspace{-5mm}
\subfigure{\includegraphics[width=0.51\columnwidth]{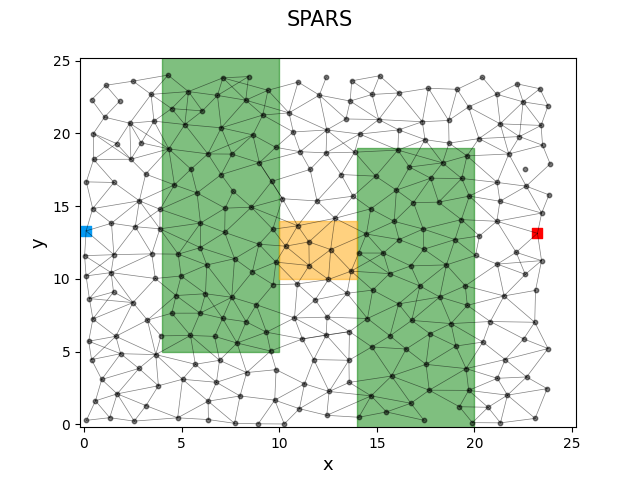}}
\caption{
Depiction of the graph topologies when using different roadmap types. The detailed explanation of the figures can be found in Figure~\ref{fig:comparison_env}. The left figure shows a grid map example used for search-based planning, and the right figure shows a roadmap example generated by SPARS used for sampling-based motion planning.
} 
\label{fig:topologies_env}
\end{figure}

\begin{figure}[!htb]
\centering
\subfigure{\includegraphics[width=0.51\columnwidth]{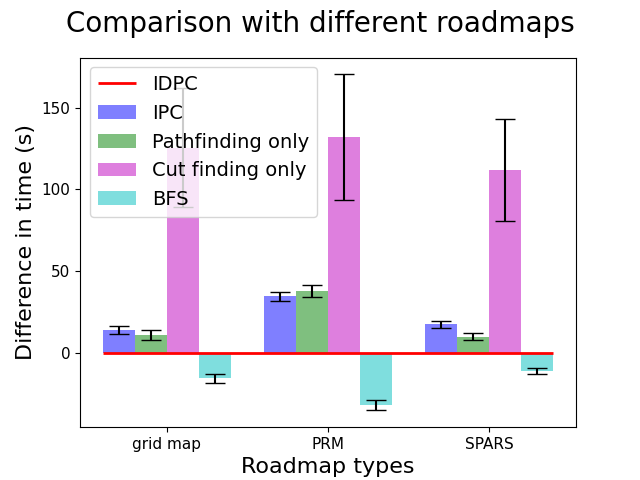}}
\hspace{-5mm}
\subfigure{\includegraphics[width=0.51\columnwidth]{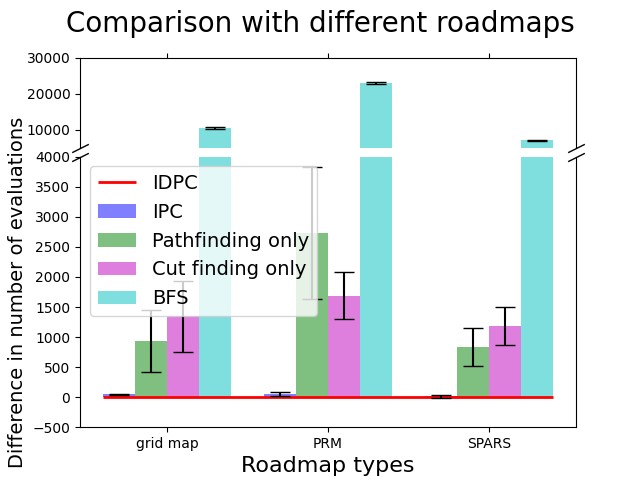}}
\caption{
Plots comparing the results when using different roadmap types. A grid map and SPARS in Figure~\ref{fig:topologies_env} and PRM in Figure~\ref{fig:comparison_env} are used for comparison. For all roadmap types, we generate roadmaps of $5000$ vertices with a different number of edges due to their different topologies. We observe that a lesser number of edges that exist in a grid map (\ie, $9660$ edges) and the SPARS roadmap (\ie, $7203$ edges) The detailed explanation of the figures can be found in Figure~\ref{fig:calibration}.
} 
\label{fig:topologies_plot}
\end{figure}

Overall, a similar trend of the previous evaluations can also be seen in Figure~\ref{fig:topologies_plot}. The pathfinding-only baseline performs relatively well on a grid map and the SPARS roadmap compared to the PRM roadmap. This is caused by a lesser number of edges that exist in the grid map and the SPARS roadmap. Nevertheless, \texttt{IDPC} still performs the best except for the completion time of the BFS-based baseline. With this result and the previous evaluations in Figure~\ref{fig:comparison_result}, we conjecture that the graph size affects the performance more significantly than the roadmap topology.

\subsection{Effect of more pathfinding executions}
\label{appen:more_pathfindings}

Since pathfinding is computationally less expensive than cut finding, we analyze the effect of increasing the number of pathfinding executions at each iteration on performance. To perform this analysis, we use \texttt{IPC}.

\begin{figure}[!htb]
\centering
\subfigure{\includegraphics[width=0.45\columnwidth]{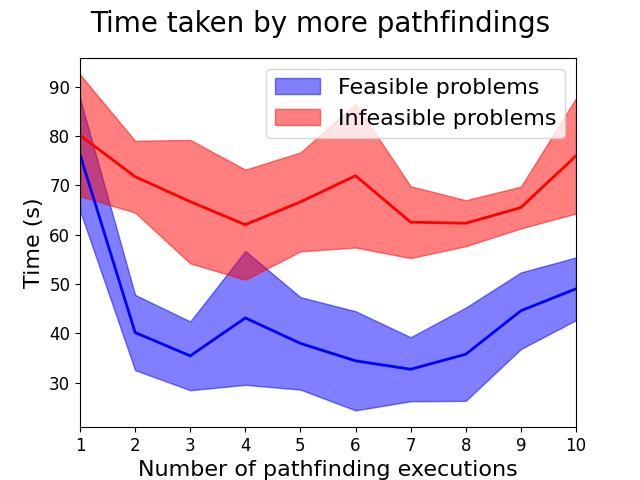}}
\subfigure{\includegraphics[width=0.45\columnwidth]{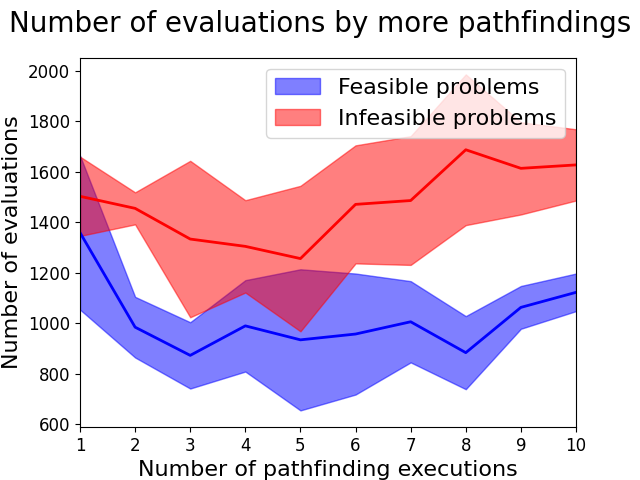}}
\caption{
Plots for the performance change as a function of the number of pathfinding executions. All results are obtained by solving twenty problem instances for each feasible and infeasible case. 
} 
\label{fig:more_pathfinding}
\end{figure}

Figure~\ref{fig:more_pathfinding} presents the performance change as the number of pathfinding executions increases. From the plots, three executions of pathfinding at each iteration perform approximately the best on average. It is, however, not statistically significantly better than other cases, as its confidence interval overlaps with others. Still, there is a trend on average that the completion time decreases at first and increases later as the number of pathfindings increases. One may consider the number of pathfinding executions at each iteration as a hyperparameter and benefit from learning the optimal number for their applications.

\clearpage
\subsection{Additional pseudo-codes for Algorithm~\ref{alg:idpc}}
\label{appen:pseudo_code}

\begin{algorithm}
\SetAlgoLined
\SetKwInOut{Input}{Input}
\SetKwInOut{Output}{Output}
\SetKwFunction{UpdateEdgeValue}{UpdateEdgeValue}
\Input{$\{\overline{G}_k\}_{k=1}^g, \widetilde{G}, P$}
\Output{$\{\overline{G}_k\}_{k=1}^g, \widetilde{G}$}

\texttt{subgraph\_ids}$\leftarrow\oldemptyset$

\For{$(v, v^\prime)\in P$} {
\texttt{is\_in\_abstract\_graph}$\leftarrow$\textsc{false}

\For{$\widetilde{v}\in \widetilde{V}$} {
\If{$\delta(\widetilde{v})=v$}{
\texttt{subgraph\_ids}$\leftarrow$\texttt{subgraph\_ids}$\cup\{\Delta(\widetilde{v})\}$

$k\leftarrow\Delta(\widetilde{v})$

\texttt{is\_in\_abstract\_graph}$\leftarrow$\textsc{true}

\Break
} % If
} % For

\If{\Not \texttt{is\_in\_abstract\_graph}} {
\texttt{subgraph\_ids}$\leftarrow$\texttt{subgraph\_ids}$\cup \{k\}$
}
} % For

\texttt{path}$\leftarrow\oldemptyset$

\For{$(v, v^\prime)\in P$} {
\If{\texttt{subgraph\_ids}$[v]\not=$\texttt{subgraph\_ids}$[v^\prime]$}{
\tcp{\texttt{subgraph\_ids}$[v]$ implies an element in \texttt{subgraph\_ids} corresponding to $v$.}

\If{\texttt{path} is collision-free} {
$c(\widetilde{e})\leftarrow$\textsc{true} \tcp{Endpoint abstract vertices of $\widetilde{e}$ in $\widetilde{G}$ correspond to $v$ and $v^\prime$.}
} % If

\texttt{path}$\leftarrow\oldemptyset$
} % If

\Else{
$k\leftarrow$\texttt{subgraph\_ids}$[v]$

$\overline{G}_k\leftarrow$\UpdateEdgeValue$(\overline{G}_k, e^k=(v, v^\prime))$ \tcp{Set $p_w$ and $p_c$ to either $0$ or $\infty$ based on the collision result.}

\texttt{path}$\leftarrow$\texttt{path}$\cup\{(v, v^\prime)\}$
} % Else
} % For

\Return{$\{\overline{G}_k\}_{k=1}^g, \widetilde{G}$, \texttt{subgraph\_ids}}

\caption{\texttt{ReflectPathEvaluation}}
\label{alg:reflect_path}
\end{algorithm}

\begin{algorithm}
\SetAlgoLined
\SetKwInOut{Input}{Input}
\SetKwInOut{Output}{Output}
\Input{$\{\overline{G}_k\}_{k=1}^g, P$, \texttt{subgraph\_ids}}
\Output{$k^*$}

\texttt{count\_collision\_edges}$\leftarrow0$

\texttt{max\_count}$\leftarrow0$

\For{$(v, v^\prime)\in P$} {
$k\leftarrow$\texttt{subgraph\_ids}$[v]$

\If{$e_k.p_c=0$} {
\tcp{$e_k=(v, v^\prime)$.} 
\texttt{count\_collision\_edges}\newline$\leftarrow$\texttt{count\_collision\_edges}$+1$
} % If
\Else{
\If{\texttt{count\_collision\_edges}\newline$>$\texttt{max\_count}} {
\texttt{max\_count}$\leftarrow$\texttt{count\_collision\_edges}

$k^*\leftarrow k$

\texttt{count\_collision\_edges}$\leftarrow0$
} % If
} % Else
} % For

\Return{$k^*$}

\caption{\texttt{ChooseSubgraph}}
\label{alg:choose_subgraph}
\end{algorithm}

\begin{algorithm}
\SetAlgoLined
\SetKwInOut{Input}{Input}
\SetKwInOut{Output}{Output}
\Input{$\overline{G}_{k^*}, \widetilde{G}$}
\Output{\texttt{substarts}, \texttt{subgoals}}

\texttt{substarts}$\leftarrow\oldemptyset$, \texttt{subgoals}$\leftarrow\oldemptyset$

\For{$\widetilde{v}\in \widetilde{V}$} {
\If{$\Delta(\widetilde{v})=k^*$} {
\If{$\tau(\widetilde{v})=$\textsc{substart}} {
\texttt{substarts}$\leftarrow$\texttt{substarts}$\cup\{\widetilde{v}\}$
} % If
\ElseIf{$\tau(\widetilde{v})=$\textsc{subgoal}} {
\texttt{subgoals}$\leftarrow$\texttt{subgoals}$\cup\{\widetilde{v}\}$
}
} % If
} % For

\For{$\widetilde{v}\in$\texttt{substarts}} {
\For{$\widetilde{v}^\prime\in$\texttt{subgoals}} {
\If{$c((\widetilde{v}, \widetilde{v}^\prime))=$\textsc{true}} {
\texttt{substarts}$\leftarrow$\texttt{substarts}$\smallsetminus\{\widetilde{v}\}$

\texttt{subgoals}$\leftarrow$\texttt{subgoals}$\smallsetminus\{\widetilde{v}^\prime\}$
} % If
} % For
} % For

\Return{\texttt{substarts}, \texttt{subgoals}}

\caption{\texttt{ClusterSubstartsAndSubgoals}}
\label{alg:clustering}
\end{algorithm}

\begin{algorithm*}
\begin{multicols}{2}
\SetAlgoLined
\SetKwInOut{Input}{Input}
\SetKwInOut{Output}{Output}
\SetKwFunction{RemoveEdge}{RemoveEdge}
\SetKwFunction{IsContain}{IsContain}
\Input{$\{\overline{G}_{k^*}, \overline{G}_{k=g+1}\}, C_{k^*}, \widetilde{G}$}
\Output{$\{\overline{G}_{k^*}, \overline{G}_{k=g+1}\}, \widetilde{G}$}

\texttt{substarts}$\leftarrow\oldemptyset$, \texttt{substarts}$^\prime\leftarrow\oldemptyset$, 
\newline
\texttt{subgoals}$\leftarrow\oldemptyset$, \texttt{subgoals}$^\prime\leftarrow\oldemptyset$

\tcp{Assign substarts and subgoals to corresponding partitions.}

\For{$v\in V^{k^*}$} {

\For{$\widetilde{v}\in\widetilde{V}$} {
\If{$\delta(\widetilde{v})=v$} {
\If{$\tau(\widetilde{v})=$\textsc{substart}} {
\texttt{substarts}\newline$\leftarrow$\texttt{substarts}$\cup\{\widetilde{v}\}$
} % If
\Else{
\texttt{subgoals}\newline$\leftarrow$\texttt{subgoals}$\cup\{\widetilde{v}\}$
} % Else

\Break
} % If
} % For
} % For

\For{$v\in V^{k=g+1}$} {

\For{$\widetilde{v}\in\widetilde{V}$} {
\If{$\delta(\widetilde{v})=v$} {
\If{$\tau(\widetilde{v})=$\textsc{substart}} {
\texttt{substarts}$^\prime$\newline$\leftarrow$\texttt{substarts}$^\prime\cup\{\widetilde{v}\}$
} % If
\Else{
\texttt{subgoals}$^\prime$\newline$\leftarrow$\texttt{subgoals}$^\prime\cup\{\widetilde{v}\}$
} % Else

\Break
} % If
} % For
} % For

\tcp{Remove first-type edges from $\widetilde{G}$.}
\For{$\widetilde{v}\in$\texttt{substarts}} {
\For{$\widetilde{v}^\prime\in$\texttt{subgoals}$^\prime$} {
$\widetilde{G}\leftarrow$\RemoveEdge$(\widetilde{G}, (\widetilde{v}, \widetilde{v}^\prime))$
} % For
} % For

\For{$\widetilde{v}\in$\texttt{substarts}$^\prime$} {
\For{$\widetilde{v}^\prime\in$\texttt{subgoals}} {
$\widetilde{G}\leftarrow$\RemoveEdge$(\widetilde{G}, (\widetilde{v}, \widetilde{v}^\prime))$
} % For
} % For

\tcp{Remove third-type edges from $\widetilde{G}$.}
\For{$\widetilde{v}\in$\texttt{substarts}} {
\For{$\widetilde{v}^\prime\in$\texttt{substarts}$^\prime$} {
$\widetilde{G}\leftarrow$\RemoveEdge$(\widetilde{G}, (\widetilde{v}, \widetilde{v}^\prime))$
} % For
} % For

\For{$\widetilde{v}\in$\texttt{subgoals}} {
\For{$\widetilde{v}^\prime\in$\texttt{subgoals}$^\prime$} {
$\widetilde{G}\leftarrow$\RemoveEdge$(\widetilde{G}, (\widetilde{v}, \widetilde{v}^\prime))$
} % For
} % For

\tcp{Add new vertices and edges \linebreak (second type) to $\widetilde{G}$.}

\For{$(v, v^\prime)\in C_{k^*}$} {
\If{\Not\IsContain$(\widetilde{G}, v)$}{
\tcp{\IsContain$(\widetilde{G}, v)=$\textsc{true} if $\widetilde{G}$ contains $v$.}
$\widetilde{V}\leftarrow\widetilde{V}\cup\{v\}$

$\delta(v)\leftarrow v$, $\Delta(v)\leftarrow k^*$,  $\tau(v)\leftarrow$\textsc{subgoal}

\For{$\widetilde{v}\in$\{\texttt{substarts}, \texttt{subgoals}\}} {
$\widetilde{E}\leftarrow\widetilde{E}\cup\{(\widetilde{v}, v)\}$

$c((\widetilde{v}, v))\leftarrow$\textsc{false}
}
} % If

\If{\Not\IsContain$(\widetilde{G}, v^\prime)$}{
$\widetilde{V}\leftarrow\widetilde{V}\cup\{v^\prime\}$

$\delta(v^\prime)\leftarrow v^\prime$, $\Delta(v^\prime)\leftarrow k=g+1$,
\newline$\tau(v^\prime)\leftarrow$\textsc{substart}

\For{$\widetilde{v}\in$\{\texttt{substarts}$^\prime$, \texttt{subgoals}$^\prime\}$} {
$\widetilde{E}\leftarrow\widetilde{E}\cup\{(\widetilde{v}, v^\prime)\}$

$c((\widetilde{v}, v^\prime))\leftarrow$\textsc{false}
}
} % If
} % For

\end{multicols}
\caption{\texttt{SubgraphPartition}}
\label{alg:partition}
\end{algorithm*}

\newpage
\subsection{\texttt{IDPC} example}
\label{appen:idpc_example}

\begin{figure*}[!htb]
\centering
\subfigure{\includegraphics[width=0.50\columnwidth]{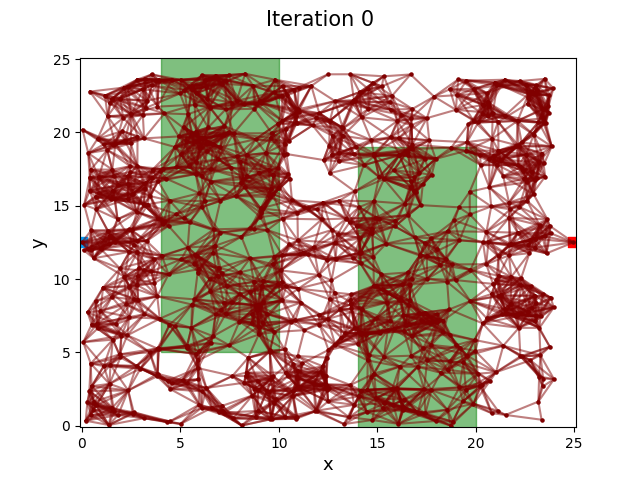}}
\hspace{-5mm}
\subfigure{\includegraphics[width=0.50\columnwidth]{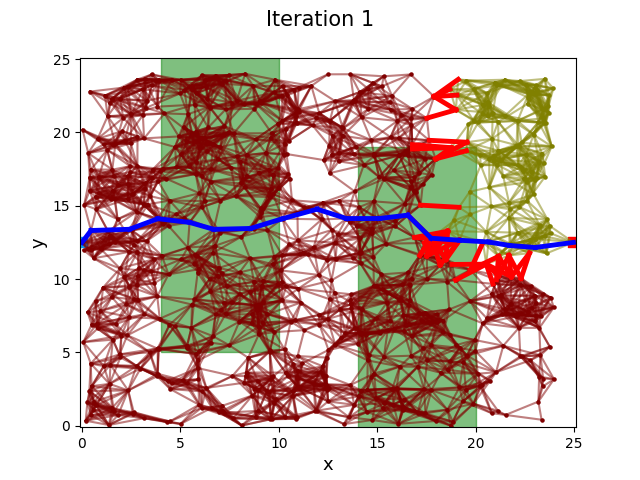}}
\hspace{-5mm}
\subfigure{\includegraphics[width=0.50\columnwidth]{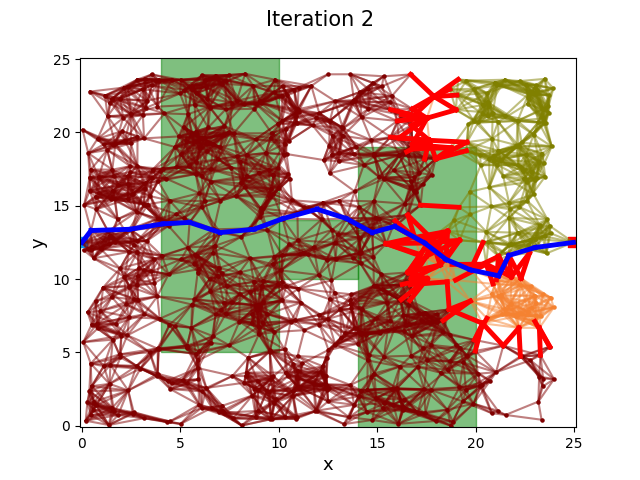}}
\hspace{-5mm}
\subfigure{\includegraphics[width=0.50\columnwidth]{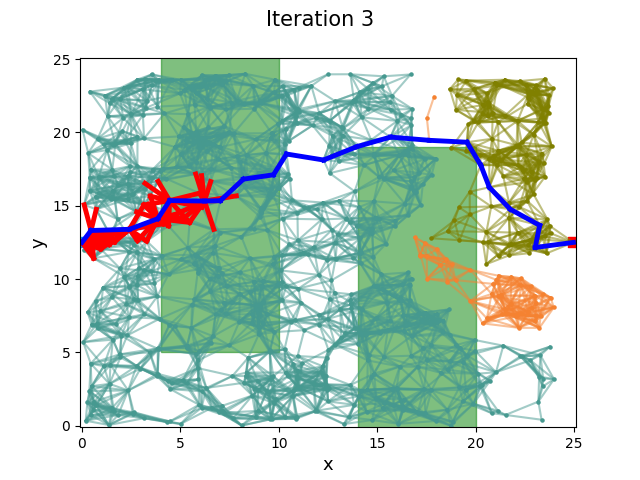}}
\subfigure{\includegraphics[width=0.50\columnwidth]{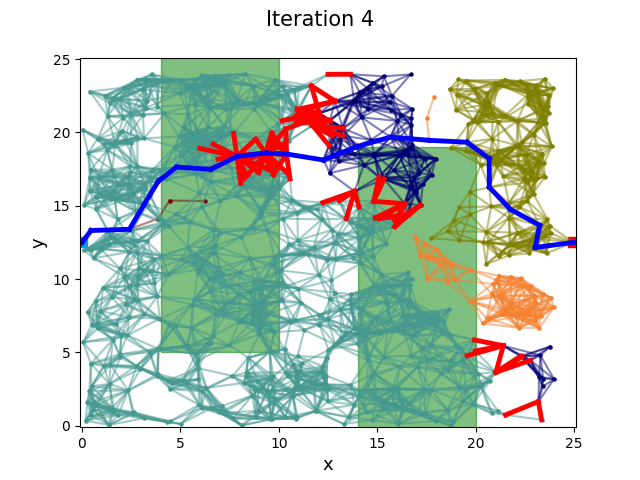}}
\hspace{-5mm}
\subfigure{\includegraphics[width=0.50\columnwidth]{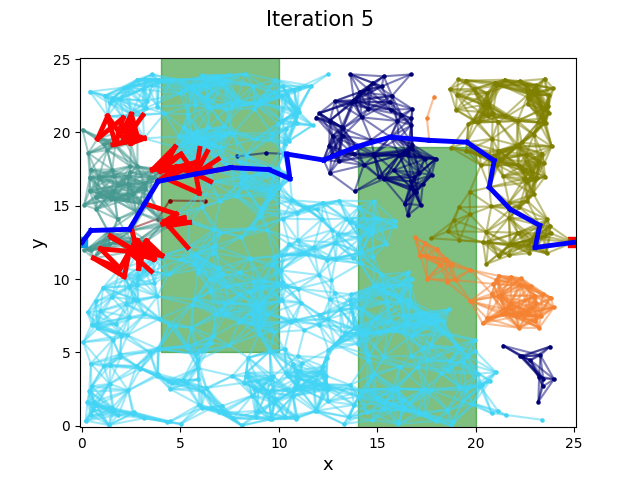}}
\hspace{-5mm}
\subfigure{\includegraphics[width=0.50\columnwidth]{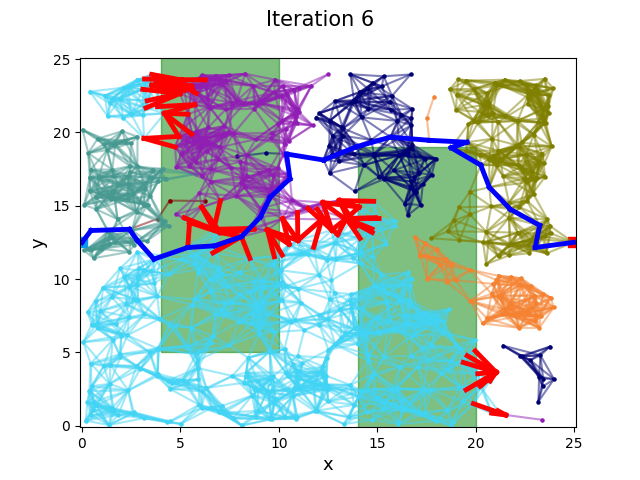}}
\hspace{-5mm}
\subfigure{\includegraphics[width=0.50\columnwidth]{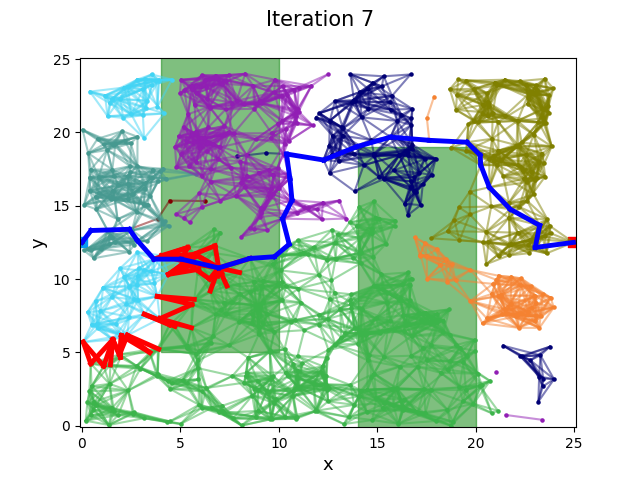}}
\subfigure{\includegraphics[width=0.50\columnwidth]{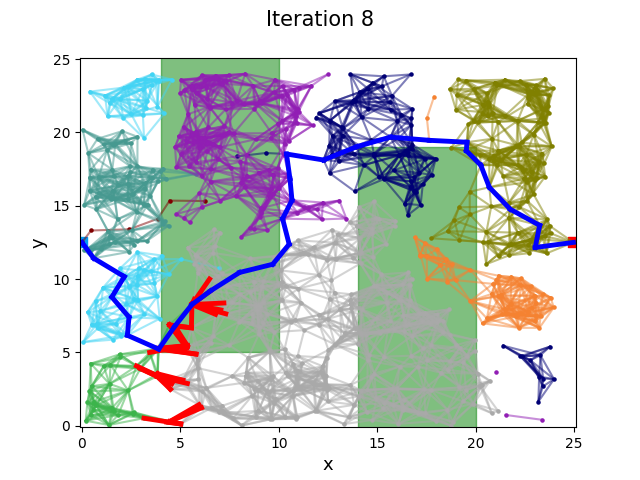}}
\hspace{-5mm}
\subfigure{\includegraphics[width=0.50\columnwidth]{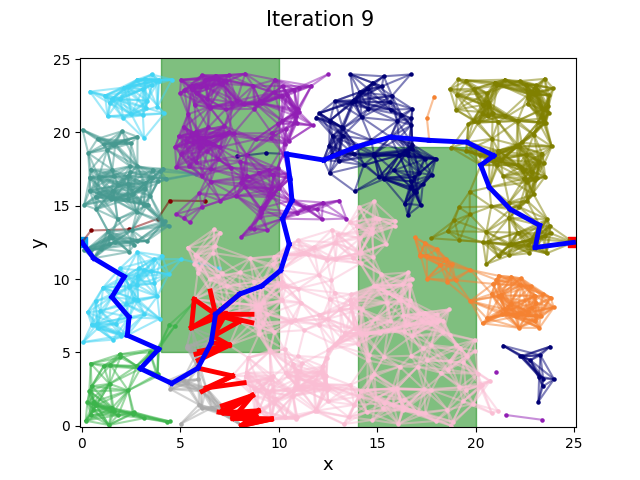}}
\hspace{-5mm}
\subfigure{\includegraphics[width=0.50\columnwidth]{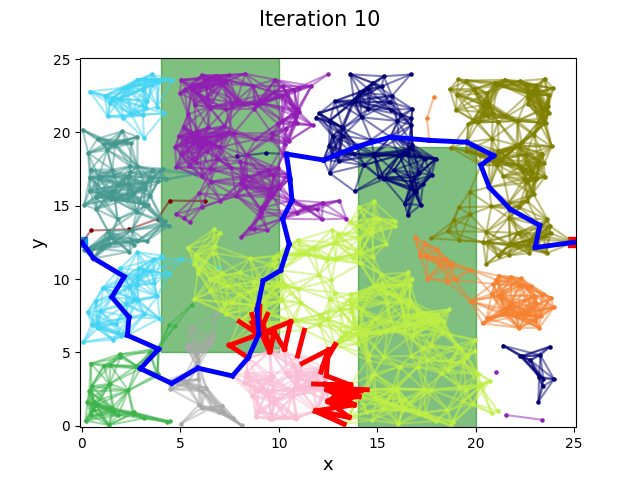}}
\hspace{-5mm}
\subfigure{\includegraphics[width=0.50\columnwidth]{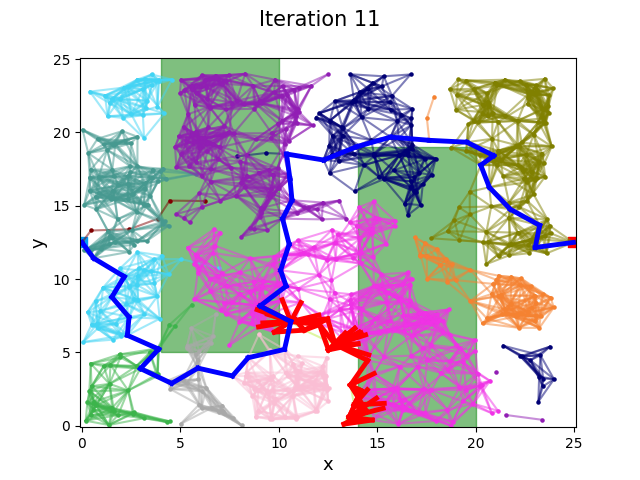}}
\subfigure{\includegraphics[width=0.50\columnwidth]{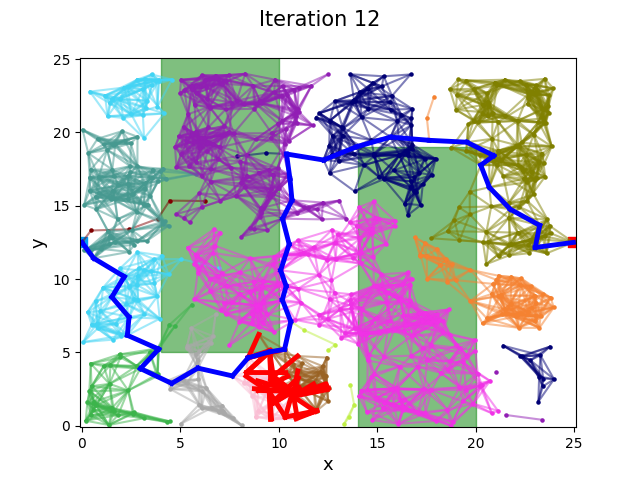}}
\caption{
Example of the \texttt{IDPC} procedure applied to a feasible problem on a prior roadmap consisting of $800$ vertices with $5000$ edges. Different colored graphs represent distinct subgraphs. The blue line and a set of red lines in each figure are a candidate path and a candidate cut found by pathfinding and cut-finding algorithms. A path is found at the $12$th iteration.
} 
\label{fig:idpc_path_example}
\end{figure*}

\begin{figure*}[!htb]
\centering
\subfigure{\includegraphics[width=0.50\columnwidth]{example_0.png}}
\hspace{-5mm}
\subfigure{\includegraphics[width=0.50\columnwidth]{example_1.png}}
\hspace{-5mm}
\subfigure{\includegraphics[width=0.50\columnwidth]{example_cut_2.png}}
\hspace{-5mm}
\subfigure{\includegraphics[width=0.50\columnwidth]{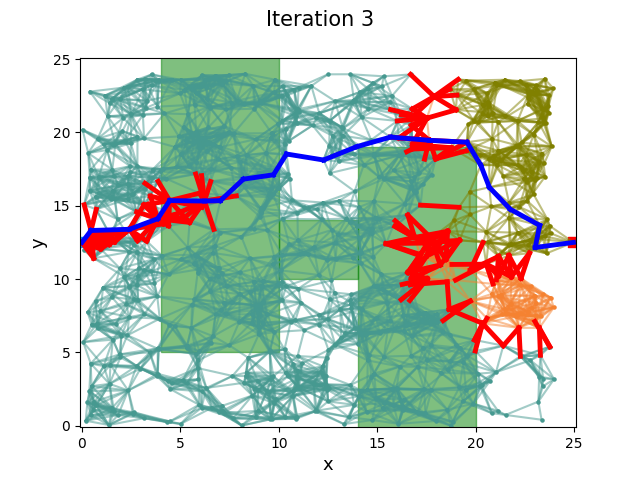}}
\subfigure{\includegraphics[width=0.50\columnwidth]{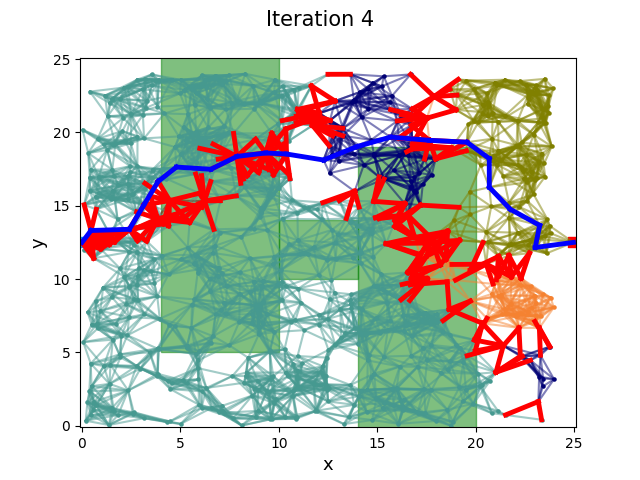}}
\hspace{-5mm}
\subfigure{\includegraphics[width=0.50\columnwidth]{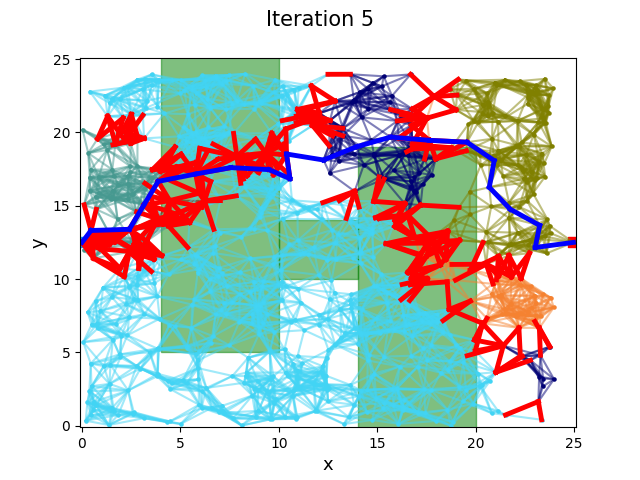}}
\hspace{-5mm}
\subfigure{\includegraphics[width=0.50\columnwidth]{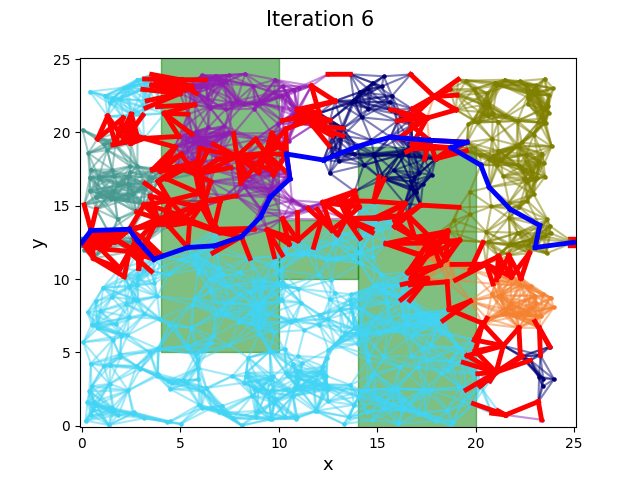}}
\hspace{-5mm}
\subfigure{\includegraphics[width=0.50\columnwidth]{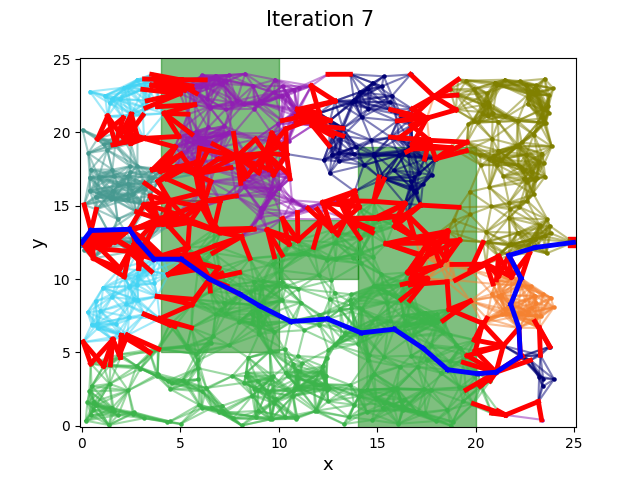}}
\subfigure{\includegraphics[width=0.50\columnwidth]{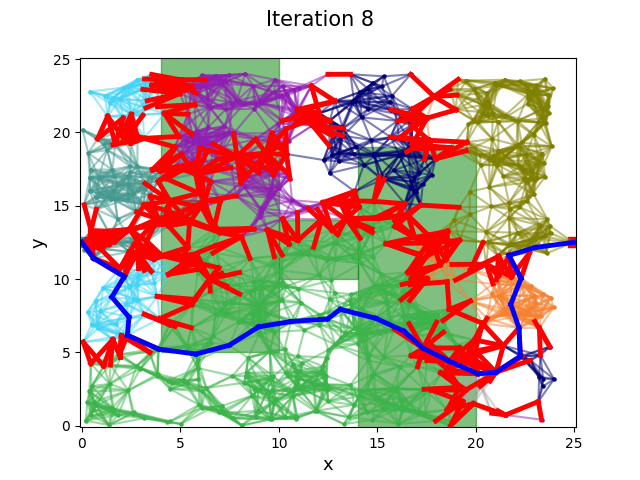}}
\caption{
Example of the \texttt{IDPC} procedure applied to an infeasible problem on a prior roadmap, with the same setting as shown in Figure~\ref{fig:idpc_path_example}. A cut is found at the $8$th iteration.
} 
\label{fig:idpc_cut_example}
\end{figure*}